\documentclass[twocolumn]{article}
\usepackage{arxiv}

\usepackage[sort&compress, numbers]{natbib}
\usepackage{soul}

\usepackage{microtype}
\usepackage{graphicx}
\usepackage{subfigure}
\usepackage{booktabs} %

\usepackage{hyperref}

\usepackage{algorithm}
\usepackage{algorithmic}
\usepackage{amsmath}
\usepackage{amssymb}
\usepackage{mathtools}
\usepackage{amsthm}
\usepackage{bm}

\usepackage{nicefrac} 
\usepackage{enumitem}
\usepackage{color}
\usepackage{multirow}
\usepackage{nccmath, amssymb}
\usepackage{bbold}
\usepackage{thmtools,thm-restate}
\usepackage{dblfloatfix}

\newtheorem{theorem}{Theorem}
\newtheorem{lemma}{Lemma}
\newtheorem{remark}{Remark}
\newtheorem{definition}{Definition}
\newtheorem{proposition}{Proposition}
\newtheorem{corollary}{Corollary}

\DeclareMathOperator*{\esssup}{\ensuremath{\text{\rm ess\,sup}}}

\DeclareMathOperator*{\argmin}{\ensuremath{\text{\rm arg\,min}}}

\DeclareMathOperator*{\range}{\ensuremath{\text{\rm Im}}}

\DeclareMathOperator*{\cl}{\ensuremath{\text{\rm cl}}}

\DeclareMathOperator{\Ker}{\ensuremath{\text{\rm Ker}}}

\DeclareMathOperator{\tr}{\ensuremath{\text{\rm tr}}}

\DeclareMathOperator*{\rank}{\ensuremath{\text{\rm rank}}}
\DeclareMathOperator*{\Span}{\ensuremath{\text{\rm span}}}

\DeclareMathOperator*{\Spec}{\ensuremath{\text{\rm Sp}}}

\newcommand{\X}{\mathcal{X}} %
\newcommand{\F}{\mathcal{F}} %
\newcommand{\RKHS}{\mathcal{H}} %

\newcommand{\cRKHSs}[1]{[\mathcal{H}]^c_{#1}} %

\newcommand{\Lii}{L^2_\im(\X)}%
\newcommand{\sigalg}{\Sigma_{\X}} %
\newcommand{\im}{\pi} %

\newcommand{\Risk}{\mathcal{R}} %
\newcommand{\ExRisk}{\mathcal{E}_{\rm HS}} %
\newcommand{\IrRisk}{\mathcal{R}_0} %
\newcommand{\error}{\mathcal{E}} %
\newcommand{\ERisk}{\widehat{\mathcal{R}}} %

\newcommand{\cRisk}{\mathcal{R}^{\circ}} %
\newcommand{\cExRisk}{\mathcal{E}_{\rm HS}^{\circ}} %
\newcommand{\cIrRisk}{\mathcal{R}_0^{\circ}} %
\newcommand{\cerror}{\mathcal{E}^{\circ}} %
\newcommand{\cERisk}{\widehat{\mathcal{R}}^{\circ}} %

\newcommand{\eimx}{\hat{\im}_x}
\newcommand{\eimy}{\hat{\im}_y}

\newcommand{\cntr}[1]{\textbf{\textsf{#1}}} %

\newcommand{\kme}[1]{k_{#1}} %
\newcommand{\fm}[1]{k_{#1}} %
\newcommand{\mt}[1]{\mu_{#1}} %
\newcommand{\emt}[1]{\widehat{\mu}_{#1}} %
\newcommand{\dt}[1]{q_{#1}} %

\newcommand{\srad}{\rho}

\newcommand{\mmd}[1]{\norm{#1}_{\RKHS^*}^2} %

\newcommand{\norm}[1]{\lVert#1\rVert}
\newcommand{\scalarp}[1]{{\langle #1\rangle}}
\providecommand{\SVDr}[1]{[\![#1]\!]_r}
\providecommand{\abs}[1]{\lvert#1\rvert}

\newcommand{\dnorm}[1]{\left\lVert#1\right\rVert_{\RKHS^*}}

\newcommand{\R}{\mathbb R}
\newcommand{\C}{\mathbb C}
\newcommand{\N}{\mathbb N}

\newcommand{\T}{\mathbb T}

\newcommand{\B}[1]{{\mathcal B}(#1)}

\newcommand{\EE}{\ensuremath{\mathbb E}}
\newcommand{\PP}{\ensuremath{\mathbb P}}

\newcommand{\Id}{I}
\newcommand{\Data}{\mathcal{D}}
\newcommand{\Koop}{{A}_{\im}}  %
\newcommand{\cKoop}{\cntr{A}_{\im}}  %
\newcommand{\adjKoop}{A^{*}_{\im}}  %
\newcommand{\adjcKoop}{\cntr{A}^{*}_{\im}}  %

\newcommand{\TrOp}{A}

\newcommand{\CME}{g_p}

\newcommand{\Estim}{G}  %
\newcommand{\cEstim}{\cntr{\Estim}}  %

\newcommand{\EEstim}{\widehat{\Estim}}  %
\newcommand{\cEEstim}{\widehat{\cEstim}}  %

\newcommand{\cHKoop}{\cEstim_{\RKHS}}  %

\newcommand{\RKoop}{\Estim_\reg}  %
\newcommand{\cRKoop}{\cEstim_\reg}  %

\newcommand{\ERKoop}{\widehat{\Estim}_\reg} %
\newcommand{\cERKoop}{\widehat{\cEstim}_\reg} %

\newcommand{\RRR}{\Estim^{\rm RRR}_{r,\reg}}  %
\newcommand{\cRRR}{\cEstim^{\rm RRR}_{r,\reg}}  %
\newcommand{\ERRR}{\EEstim^{\rm RRR}_{r,\reg}}  %
\newcommand{\cERRR}{\cEEstim^{\rm RRR}_{r,\reg}}  %

\newcommand{\EPCR}{\EEstim^{\rm PCR}_{r,\reg}}  %
\newcommand{\cEPCR}{\cEEstim^{\rm PCR}_{r,\reg}}  %

\newcommand{\Cx}{C} %
\newcommand{\cCx}{\cntr{\Cx}} %

\newcommand{\ECx}{\widehat{\Cx}} %
\newcommand{\cECx}{\widehat{\cCx}} %

\newcommand{\Cxy}{T}  %
\newcommand{\cCxy}{\cntr{\Cxy}}  %
\newcommand{\ECxy}{\widehat{\Cxy}}  %
\newcommand{\cECxy}{\widehat{\cCxy}}  %

\newcommand{\ECyx}{\ECyx^*}  %

\newcommand{\Creg}{\Cx_\reg}
\newcommand{\cCreg}{\cCx_\reg}
\newcommand{\ECreg}{\ECx_\reg}
\newcommand{\cECreg}{\cECx_\reg}

\newcommand{\Kx}{K_x} %
\newcommand{\cKx}{\cntr{K}_x}
\newcommand{\Kreg}{K_\reg} %
\newcommand{\cKreg}{\cntr{K}_\reg}
\newcommand{\Ky}{K_y} %
\newcommand{\cKy}{\cntr{K}_y}
\newcommand{\Kxy}{K_{xy}}

\newcommand{\Kxz}{K_{xz}}

\newcommand{\TS}{S_\im}  %
\newcommand{\cTS}{\cntr{S}_\im}  %

\newcommand{\cTSs}[1]{\cntr{S}_{#1}}  %

\newcommand{\TZ}{\Koop\TS}  %
\newcommand{\cTZ}{\cKoop\cTS}

\newcommand{\ES}{\widehat{S}_x} %
\newcommand{\cES}{\widehat{\cntr{S}}_x} %
\newcommand{\EZ}{\widehat{S}_y} %
\newcommand{\cEZ}{\widehat{\cntr{S}}_y} %

\newcommand{\TB}{B}  %
\newcommand{\EB}{\widehat{B}} %

\newcommand{\TP}{P}  %
\newcommand{\EP}{\widehat{P}} %

\newcommand{\HS}[1]{{\rm{HS}}\left(#1\right)} %
\newcommand{\hnorm}[1]{\norm{#1}_{\rm{HS}}}

\newcommand{\transitionkernel}{p} %
\newcommand{\reg}{\gamma}
\newcommand{\rate}{\varepsilon}

\newcommand{\rpar}{\alpha}
\newcommand{\spar}{\beta}
\newcommand{\epar}{\tau}
\newcommand{\rcon}{a}
\newcommand{\scon}{b}
\newcommand{\econ}{c_\epar}
\newcommand{\bcon}{c_\RKHS}

\newcommand{\pcon}{p}
\newcommand{\sumcon}{s}

\newcommand{\kreiss}{\eta}
\newcommand{\dtoins}{d}

\newcommand{\refun}{\psi}
\newcommand{\erefun}{\widehat{\psi}}
\newcommand{\lefun}{\xi}
\newcommand{\elefun}{\widehat{\xi}}

\newcommand{\eval}{\lambda}

\newcommand{\eeval}{\widehat{\eval}}

\newcommand{\one}{\mathbb 1}

\newcommand{\cf}{\one_\im}
\newcommand{\cv}[1]{\one_{#1}}
\newcommand{\cvn}{\cv{n}}

\newcommand{\V}{V}
\newcommand{\fH}{\phi}
\newcommand{\fG}{\psi}

\newcommand{\HSr}{{\mathcal B}_r(\RKHS)}

\newcommand{\J}{J_\im}
\newcommand{\Jn}{J_n}

\newcommand{\embedding}{k}

\newcommand{\MMD}{\text{MMD}}

\newcommand{\SM}{\mathcal{M}^{+}(\X)}

\newcommand{\TM}{\mu}
\newcommand{\EM}{\hat{\TM}}
\newcommand{\density}{\varphi}

\newcommand{\err}{\text{E}}
\newcommand{\DRRR}{DRRR}

\usepackage[textsize=tiny]{todonotes}
\newcommand{\prune}[2][]{\todo[color=orange!20,#1]{{\bf PI:} #2}}

\newcommand{\MP}[1]{\textcolor{brown}{#1}}

\newcommand{\karim}[2][]{\todo[color=blue!20,#1]{{\bf KL:} #2}}

\begin{document}
\date{}
\title{Consistent Long-Term Forecasting of Ergodic Dynamical Systems}
\author{Prune Inzerilli \thanks{} \\ \'Ecole Polytechnique\\
\texttt{prune.inzerilli@polytechnique.edu} \And Vladimir R. Kostic\\
Istituto Italiano di Tecnologia\\
University of Novi Sad\\
\texttt{vladimir.kostic@iit.it}
\And
Pietro Novelli\\
Istituto Italiano di Tecnologia\\
\texttt{pietro.novelli@iit.it}
\And
Karim Lounici\\
CMAP \'Ecole Polytechnique\\
\texttt{karim.lounici@polytechnique.edu}
\And
Massimiliano Pontil \\
Istituto Italiano di Tecnologia\\
University College London\\
\texttt{massimiliano.pontil@iit.it}}

\twocolumn[

\maketitle

\begin{abstract}
    We study the evolution of distributions under the action of an ergodic dynamical system, which may be stochastic in nature. By employing tools from Koopman and transfer operator theory one can evolve any initial distribution of the state forward in time, and we investigate how estimators of these operators perform on long-term forecasting. {Motivated by the observation that standard estimators may fail at this task,} we introduce a learning paradigm that neatly combines classical techniques of \textit{eigenvalue deflation} from operator theory and \textit{feature centering} from statistics. This paradigm applies to any operator estimator based on empirical risk minimization, making them satisfy learning bounds which hold \textit{uniformly} on the entire trajectory of future distributions, and abide to the \textit{conservation of mass} for each of the forecasted distributions. Numerical experiments illustrates the advantages of our approach in practice.
    \end{abstract}
\bigskip
\bigskip
]
{
  \renewcommand{\thefootnote}%
    {\fnsymbol{footnote}}
  \footnotetext[1]{Work done while at Istituto Italiano di Tecnologia.}
}

\section{INTRODUCTION
}

Dynamical systems offer a mathematical framework to describe the evolution of state variables over time. In many applications, these models, often represented by unknown nonlinear differential equations (ordinary or partial, and possibly stochastic), necessitate the use of data-driven techniques for characterizing the dynamical system and forecasting future states. This task has garnered substantial interest in recent decades due its application in many fields, including energy forecasting~\cite{mohan2018data}, epidemiology~\cite{proctor2015discovering} 
finance~\cite{Pascucci2011}, atomistic simulations~\cite{Schutte2001}, fluid dynamics \cite{mezic2013analysis}, weather and climate forecasting \cite{scher2018toward}, and many more. 

Particular emphasis is placed on long-term forecasting, which, given any initial state distribution, aims to predict how it (or a given statistics thereof) will evolve over time, until a long-term horizon. 
The accuracy of long-term forecasting is of utmost importance for effective strategic planning and early warning systems. However  structured data modalities, an increasing volume of observations, and highly non-linear relationships among covariates pose significant challenges to current long-term forecasting approaches. {In this work, we specifically address the problem of long-term forecasting for ergodic dynamical systems, whose states converge to an unknown but invariant distribution over time.}

The Koopman\footnote{Historically, the Koopman operator was introduced for deterministic dynamical systems, while the transfer operator is its analogue in the stochastic case. The results presented in this paper, however, apply to both settings.} operator regression (KOR) framework to learn dynamical systems from data %
{became popular in the last few years as it enables its users to accomplish several important tasks} including interpretation, control and forecasting; see, for example, the monographs~\cite{Brunton2022,Kutz2016} for an introduction to these topics. {The very same framework is also used to forecast state \textit{distributions} by means of the duality relation connecting the Koopman operator (which evolves states and observables) and the Perron-Frobenius operator, which evolves distributions.} 
Kernel based algorithms to learn the Koopman or transfer operators have been proposed in  %
~\cite{Alexander2020,Bouvrie2017,Das2020,Klus2019,Kostic2023b,Kostic2022, Meanti2023, OWilliams2015, Bevanda2023, Hou2023}. 
Deep learning approaches, {on the other hand}, were explored in the works~\cite{Bevanda2021,Fan2021,Lusch2018, Azencot2020, Morton2018}.

Our aim it to improve over mainstream KOR estimators, whose performance is known to deteriorate as the forecasting horizon extends further into the future \cite{Kostic2022}. Within the setting of \textit{ergodic} dynamical systems, arbitrary initial state distributions are bound to converge to an unknown but unique \textit{invariant}\footnote{That is, invariant under the action of the Perron-Frobenius operator.} distribution. In this work we propose a paradigm to inject this prior knowledge into existing kernel-based algorithms. This, in turn, has the effect of producing estimators which more accurately forecast the long-term behavior of the state distribution than state-of-the art approaches \cite{Brunton2022,Kostic2022} that do not leverage this information. Our experiment in Fig.~\ref{fig:rel_MMD}, indeed, shows a decline in the distribution forecasting accuracy of classical estimators as the forecasting horizon extends. In contrast, when the same estimator is augmented by our \textit{deflate-learn-inflate} (DLI) paradigm, the predicted distribution flow is substantially more accurate.

\begin{figure}
  \centering
  \includegraphics[width=0.95\columnwidth]{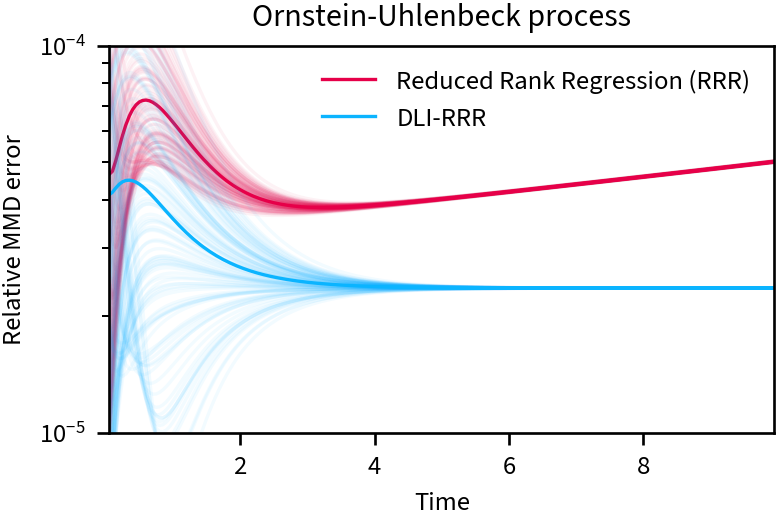} %
  \caption{\textit{Distribution forecasting}: relative MMD error for the Ornstein-Uhlenbeck process for 100 independent experiments (thin lines).}
  \label{fig:rel_MMD} %
\end{figure}

\paragraph{Contributions} We make the following contributions: 
(i) we introduce a novel distribution forecasting paradigm, dubbed deflate-learn-inflate (DLI), ensuring consistent forecasting quality across predicted future distributions, even for an infinite-time horizon. 
{This approach combines tools from kernel mean embeddings to capture the complex relationships between data points and their distributions, with classical techniques of \textit{deflation} of eigenvalues from operator theory, \textit{centering} features from statistics and the fundamental physics concept of \textit{preservation of mass}.
This paradigm can be seamlessly integrated into mainstream KOR estimators to enhance their long-term forecasting accuracy; 
(ii) 
We provide non-asymptotic bounds on the maximum mean discrepancy distance between the true state distributions and their forecasts which hold uniformly over the complete trajectory of future states even for infinite horizon; (iii) {We confirm our theoretical claims by testing the DLI paradigm on two different experimental settings.} %

\paragraph{Paper Organization} In Sec.~\ref{sec:koopman_framework}, we set the stage by recalling the Koopman/transfer operator framework and the use of duality to forecast state distributions.  %
In Sec.~\ref{sec:dli}, we present the DLI paradigm {and discuss its implementation, while} Sec.~\ref{sec:theory} contains our theoretical guarantees. Finally, in Sec.~\ref{sec:exp} we {test the DLI approach on two different settings, 
highlight its improved long-term forecasting performance.}

\paragraph{Notations} For two measures $\mu$ and $\nu$, $\mu\ll\nu$ means that $\mu$ is absolutely continuous w.r.t. $\nu$, in which case $d\mu/d\nu$ denotes the Radon-Nikodym derivative. Given a separable Hilbert space $\mathcal{H}$, we let $\HS{\mathcal{H}}$ be the Hilbert space of Hilbert-Schmidt (HS) operators on $\mathcal{H}$ endowed with the norm  $\hnorm{A}^{2} \equiv\sum_{i \in \mathbb{N}} \Vert Ae_{i} \Vert_{\mathcal{H}}^{2}$, for $A \in \HS{\mathcal{H}}$, where $(e_{i})_{i\in \mathbb{N}}$ is an orthonormal basis of $\mathcal{H}$. For any bounded operator $A$ on $\RKHS$, we denote by $\srad(A)$ and $\norm{A}$ the spectral radius and operator norm of $A$ respectively. Note that for any bounded operator $A$, we have $\rho(A)\leq \norm{A}$, see, e.g., \cite{TrefethenEmbree2020}.

\section{DISTRIBUTION FORECASTING}
\label{sec:koopman_framework}

In this section, we give some background on the Koopman/transfer operator framework, introduce empirical estimators and then present the kernel mean embedding approach to describe the flow of state distributions.

\paragraph{Background.}

Dynamical systems are a mathematical framework to describe the evolution of a state variable $x_t \in \X$ over time $t$. In this work, we focus on discrete systems, that is the time $t$ is an integer, so that the future state of the system $x_{t+1}$ depends on the current one $x_t$ via an equation, which is in general non-nonlinear and  may be either deterministic or stochastic. In this sense, a dynamical systems can be described by a Markov time homogeneous stochastic process $(X_t)_{t \in \N}$~\cite{Lasota1994}, where the random variables $X_t$ take values in a measurable space $\X$, endowed with $\sigma$-algebra $\Sigma_\X$. The Markov property means that
$$
\PP[X_{t+1}\,\vert\,(X_s)_{s=0}^t] = \PP[X_{t+1}\,\vert\,X_t],
$$ 
while time-homogeneity requires the r.h.s.~probability to be time independent,~that is there exists $\transitionkernel\colon \X \times \sigalg \to [0,1]$, called {\it transition kernel}, such that, for every $(x, B) \in \X \times \sigalg$ and $t\in \N$,
\[
\PP[X_{t+1} \in B \,\vert\,X_t=x] = \transitionkernel (x,B).
\]
We further assume that the dynamical system is {\em ergodic}, that is, there exists a \textit{unique} probability distribution $\im$, called \textit{invariant measure}, such that if $X_0\sim\im$, then $X_t\sim\im$, for every $t\in\mathbb{N}$. 

Ergodic dynamical systems %
are general enough to capture several important phenomena, including (discretized) Langevin dynamics~\cite{Davidchack2015} or other systems constructed from the discretization of stochastic differential equations. They can be studied via Markov operators, and, in particular with \emph{forward transfer operators} $\TrOp_t\colon\Lii\to\Lii$ defined on the space $\Lii$ formed by  square integrable functions w.r.t. the invariant measure as
\begin{equation}\label{eq:transfer-operators}
[\TrOp_{t}f](x) := \EE[ f(X_{t})\,\vert X_{0} = x],\;x\in\X, t\in\mathbb{N}. 
\end{equation}   
We refer the reader to \cite{Kostic2022} for a more detailed discussion and a proof of the fact that these operators are well defined on $\Lii$.

Due to their prominence in the data-driven (deterministic) dynamical systems community \cite{Kutz2016,Brunton2022}, we will also call $\TrOp_t$ the (stochastic) Koopman operators. These operators at different time steps are related through the Chapman-Kolmogorov equation $\TrOp_{t + s} \,{=}\, \TrOp_{t}\,\TrOp_{s}$, implying that the family $(\TrOp_{t})_{t\in\mathbb{N}}$ forms a semigroup~\cite{arnold1974}.~Notice that since we work with discrete processes it holds that $\TrOp_t \,{=}\, \Koop^t$, where with some abuse of notation we have defined $\Koop\,{=}\,\TrOp_1$, to highlight the dependence of the Koopman operator on the invariant distribution $\pi$.  

The significance of Koopman operators lies in their ability to effectively \textit{linearize} the underlying Markov processes. In particular in this work we are interested in using these operators to characterize the flow $(\mu_t)_{t \in \N}$ of the dynamical system, where $\mt{t}$ is the law (or marginal distribution) of $X_t$. 
To simplify the discussion, 
it is convenient to assume that each distribution $\mt{t}$ is absolutely continuous w.r.t. the invariant measure $\im$, that is, it has a density $\dt{t}:=d\mt{t} / d\im \in L^1_\im(\X)$ defined via the Radon-Nikodyn derivative. If in addition the densities are square-integrable, then, for every $\dt{0}\in\Lii$, the flow of the probability distributions $(q_t)_{t \in \N}$ follows \textit{linear dynamics} in the space $\Lii$, given by the equations
\begin{equation}
q_t = \TrOp_t^* q_{0} = (\Koop^*)^t q_0,\quad t \in \N.
    \label{eq:density_flow}
\end{equation}
The operator $\Koop^*$, known as {\it Perron-Frobenius operator}, is the adjoint of the Koopman operator, and it is given, for every $q \in \Lii$, by 
\begin{align}
\label{eq:Astar_def}
[\Koop^* q](y) = \int \transitionkernel^*(y,dx) q(x)
\end{align}
where $p^*$ is the time-reversal transition kernel, defined as 
$p^*(y,B) \,{=}\, \PP[X_{t-1}\,{\in}\, B|X_t \,{=}\, y]$, for every $B \,{\in}\, \Sigma$ and $y \,{\in}\, \X$, see \cite[App.~A]{Kostic2022} for details.
An implication of \eqref{eq:density_flow} is that, once linearized, the process can be efficiently evolved from any initial density $q_{0}$ using the spectral theory of bounded operators. In particular, \eqref{eq:transfer-operators} and \eqref{eq:Astar_def} imply %
that 
\begin{equation}\label{eq:specrad}
\TrOp_{\im}\cf = \cf,\quad\text{ and } \quad\TrOp_{\im}^*\cf = \cf,\;  %
\end{equation}
where $\cf$ is the function in $\Lii$ which is almost everywhere equal to $1$. Note that, since $\norm{\TrOp_{\im}}=1$, its largest eigenvalue is equal to $1$. 
Moreover, since the process is ergodic, $\cf$ is the unique (up to scaling) eigenfunction associated to this eigenvalue~\cite{Lasota1994}.
In this work we aim to learn the Koopman operator so that the flow $(q_t)_{t \in \N}$ can be reliably forecasted from an initial distribution $q_0$. {We now discuss how Koopman operators can be learned from data, and next
we clarify how the same machine learning approaches can be repurposed to forecast \textit{distributions}.}

\paragraph{Koopman Operator Regression.} 
In recent years, the abundance of emerging machine learning algorithms has sparked a growing interest on data-driven dynamical systems. In this setting $\Koop$ is not known, and a key challenge is to learn it from data. 
An appealing class of KOR learning algorithms~\cite{Brunton2022,Kutz2016} aim to learn the Koopman operator on a predefined reproducing kernel Hilbert space (RKHS) $\RKHS$ consisting of functions from $\Lii$. Namely, let 
$$
k:\X\times\X \to \R
$$ 
be a symmetric and positive definite kernel function and $\RKHS$ the corresponding RKHS~\cite{aron1950}, with norm denoted as $\|\cdot \|_{\RKHS}$. We let $x\mapsto\fm{x}\equiv k(\cdot,x) \in \RKHS$ denote the {\em canonical feature map}~and assume, for every $x \in \X$, that $k(\cdot,x) \in \Lii$, which in turn implies that $\RKHS \subset \Lii$;~for an introduction to RKHS see, e.g.,~\cite[Chapter 4.3]{Steinwart2008}.

The spaces $\RKHS$ and $\Lii$ have, however, different norms. 
To handle this ambiguity, we use the {\em injection operator} $\TS:\RKHS \hookrightarrow \Lii$ such that, for all $f \in \RKHS$, the object $\TS f$ is the element of $\Lii$  which is pointwise equal to $f \in \RKHS$, but endowed with the appropriate $\Lii$ norm. Moreover, the adjoint of the injection is given, for every $f\in\Lii$, by
\begin{equation}
\label{eq:adjS}
\TS^*f = \EE_{X\sim\im}[f(X)\fm{X}]\in\RKHS. 
\end{equation}

In data-driven dynamical systems we are provided with a dataset\footnote{For simplicity we consider the i.i.d. setting, however our forthcoming analysis is directly applicable to sample trajectories; see \cite{Kostic2022} for a discussion.} $\Data_n := (x_{i}, y_{i})_{i = 1}^{n}$ of consecutive states 
sampled at equilibrium, 
and wish to learn the Koopman operator by minimizing the \textit{empirical risk}
\begin{equation}\label{eq:empirical_risk}
    {\ERisk}(\Estim) := \textstyle{\tfrac{1}{n}\sum_{i \in[n]}} \Vert\fm{y_{i}} - \Estim^{*}\fm{x_i}\Vert^{2},
\end{equation}
over operators $\Estim\colon \RKHS\to\RKHS$. 
Defining the sampling operators $\ES,\EZ\colon \RKHS\to\R^n$ as $\ES h:=\tfrac{1}{\sqrt{n}} (h(x_i))_{i\in[n]}$ and $\EZ h:=\tfrac{1}{\sqrt{n}} (h(y_i))_{i\in[n]}$, \eqref{eq:empirical_risk} can be equivalently written as 
$$
{\ERisk}(\Estim) = \Vert \EZ-\ES \Estim \Vert^2_{\rm HS}
$$
from which it is apparent that the estimators are of the form $\EEstim= \ES^* W \EZ$, for some $n \times n$ real matrix $W$. In particular, 
we mention three important estimators that are often used in applications: kernel ridge regresssion (KRR), principal component regression (PCR) and reduced rank regression (RRR), implementing different forms of regularization on the operator $G$, see \cite{Kostic2022} for more information and the explicit form of the estimators involving kernel matrices. KRR is defined as
\begin{equation}\label{eq:empirical_KRR_estimator}
 \ERKoop= \ECx_\reg^{-1} \ECxy,
\end{equation}
where $\ECx_{\reg} := \ECx + \reg\Id_{\RKHS}$, $\ECx= \ES^* \ES$ is the empirical covariance of the input, and $\ECxy= \ES^* \EZ$ is the empirical cross-covariance between input and output. PCR is given by 
\begin{equation}\label{eq:empirical_PCR_estimator}
 \EPCR = \SVDr{\ECreg^{-1}}\ECxy,
\end{equation}
where $\SVDr{\cdot}$ denotes the $r$-truncated SVD. 
The RRR estimator on the other hand is given by
\begin{equation}\label{eq:empirical_RRR_estimator}
   \ERRR = \ECx_{\reg}^{-1/2}\SVDr{\ECx_{\reg}^{-1/2}\ECxy}.
\end{equation}
From the estimator $\EEstim$ we can build an estimator of  $\TrOp_{\vert_{\RKHS}}\colon\RKHS\to\Lii$, the {\em restriction} of the Koopman operator to the RKHS, namely   $\TrOp_{\vert_{\RKHS}}=\Koop\TS$. Specifically, we estimate $\TrOp_{\vert_{\RKHS}}$ by the operator $\TS\EEstim\colon \RKHS\to\Lii$. The corresponding  estimation quality can be measured by the operator norm error
\begin{equation}
    \label{eq:Risk}
    \error(\EEstim):=\Vert
\Koop\TS-
\TS\EEstim\Vert,
\end{equation}
see~\cite{Kostic2023b} for a discussion of these ideas. Therefore, to properly learn dynamics, $\RKHS$ should be chosen so that: i) $\TrOp_{\vert_{\RKHS}}$ approximates well $\Koop$, i.e. the space $\RKHS$ needs to be big enough relative to the domain of $\Koop$; ii) the risk ${\ERisk}(\EEstim)$ 
should be small enough. When $\RKHS$ is an infinite-dimensional {\em universal} RKHS both requirements are satisfied~\cite{Kostic2022}, i.e. $\RKHS$ is dense in $\Lii$ and the error \eqref{eq:Risk} converges to zero w.r.t. sample size. 
\paragraph{Kernel Mean Embedding of the Flow.}
The RKHS framework naturally allows one to introduce a metric on the space of signed measures $\SM$. Specifically, the dual norm
\begin{equation}\label{eq:hnorm_star}
\dnorm{\mu}:=\sup_{\norm{h}_{\RKHS}\leq1}\int_{\X}h(x)\mu(dx),\;\mu\in\SM, 
\end{equation}
induces the weak$^*$ topology on $\SM$~\cite{Muandet2017}. The supremum above is attained at 
\begin{equation}\label{eq:kme}
\kme{\mu}:=\int_{\X}\fm{x}\mu(dx)=\EE_{X\sim\mu}[\fm{X}]\in\RKHS,
\end{equation}
which is known as the \textit{kernel mean embedding} (KME) of the measure $\mu$. Furthermore, for two signed measures $\mu$ and $\nu$ the square distance of the corresponding kernel mean embeddings,
\begin{equation}\label{eq:mmd}
\mmd{\mu-\nu}=\norm{\kme{\mu}-\kme{\nu}}^2_{\RKHS},\;\mu,\nu\in\SM,
\end{equation}
is called the \textit{maximum mean discrepancy} (MMD)~\cite{Muandet2017}.

Recalling that $\dt{t}=d\mt{t}/d\im$, and using the adjoint of the injection operator \eqref{eq:adjS} we have that $\kme{\mt{t}}=\TS^*\dt{t}$, $t\in\mathbb{N}$.  This way, we can transform the flow \eqref{eq:density_flow} in $\Lii$ to the flow $(\fm{\mu_t})_{t \in \N}$ in $\RKHS$. An advantage of this is that the latter can be estimated. Indeed, since $\TS \EEstim$ aims to estimate $\Koop \TS$, iterating the approximation $\TS^*\adjKoop \approx \EEstim^* \TS^*$, we see that, for every $t\in\mathbb{N}$,
$$
\kme{\mt{t}} =\TS^*(\adjKoop)^t \dt{0} \approx (\EEstim^*)^{t}\TS^*\dt{0} = (\EEstim^*)^{t} \kme{\mt{0}}.
$$
The r.h.s.~of the above equation gives the estimated flow of the process in the RKHS, 
\begin{equation}\label{eq:rkhs_flow}
\kme{\emt{t}}:=(\EEstim^*)^{t}\,\kme{\mt{0}}\in\RKHS,\;t\in\mathbb{N}.
\end{equation}
Notice that a property of empirical risk minimization estimators is that $\range(\EEstim^*)\subset\Span(\fm{y_i})_{i\in[n]}$. Consequently, $\kme{\emt{t}}$ is the embedding of a measure  $\emt{t} =\sum_{i\in[n]} w_{t,i}\delta_{y_i}$, for some  $w_t\in\R^n$.  Moreover, the MMD estimation error between the true and empirical measure is given by
\begin{equation}\label{eq:flow_error}
\dnorm{\emt{t}-\mt{t}}^2 \!=\! \Vert[(\EEstim^*)^{t}\TS^* \! - \! \TS^*(\adjKoop)^t]\dt{0}\Vert_{\RKHS}^2,
\end{equation}
that is $\dnorm{\emt{t}-\mt{t}}\leq \error_t(\EEstim)\norm{\dt{0}}_{\Lii}$, where we have defined the operator norm error after $t$ steps as $\error_t(\EEstim)= \Vert\Koop^t \TS - \TS \EEstim^{t}\Vert$. In particular, note that  $\error_1(\EEstim)=\error(\EEstim)$.

\begin{remark}[Impossibility of Long-Term Forecasting]
\label{rem:impossibility}
Even for ergodic dynamical systems considered in this work,
we cannot in general guarantee the consistency of the trajectories \eqref{eq:density_flow} and \eqref{eq:rkhs_flow}, that is, we cannot recover the dynamics properly for large time-horizons. To clarify this, we consider estimator $\EEstim$ admitting spectral decomposition $(\eeval_j,\widehat{\lefun}_j,\widehat{\refun}_j)_{j=1}^{n}$ where $\eeval_j\in \C$ and $\widehat{\lefun}_j$ and $\widehat{\refun}_j$ are complex-valued function with components in $\RKHS$. For most estimators $\EEstim$, we would have either $|\eeval_1|<1$ or $|\eeval_1|>1$. In the former case, we deduce that for any $h\in \RKHS$, we have $\langle \kme{\emt{t}}, h \rangle\rightarrow 0$ as $t\rightarrow \infty$, i.e, $\emt{t}\overset{\ast}{\rightharpoonup}0$ as $t\rightarrow \infty$ (weak topology convergence). Conversely when $|\eeval_1|>1$, the empirical dynamic can diverge while the true one remains bounded. 
{Indeed, while $\mt{t}\to\im$ for $t\to\infty$, we have that  $\dnorm{\emt{t}}\!\geq\! |\eeval_1|^{t} \norm{\elefun_1}\abs{\langle \kme{\mt{0}} , \erefun_1\rangle} - \sum_{j\geq2} |\eeval_j|^{t} \norm{\elefun_j}\abs{\langle \kme{\mt{0}} , \erefun_j\rangle}\to\infty$, as soon as $\langle \kme{\mt{0}} , \widehat{\refun}_1\rangle\neq 0$.} 
\end{remark}
A possible heuristic to address the above issue is to round the largest eigenvalue of the estimator $\EEstim$ to $1$. However, this approach proves to be overly simplistic and ineffective in practice. In the next section we discuss a  principled approach to solve this problem.

\section{DEFLATE-LEARN-INFLATE PARADIGM}
\label{sec:dli}
In this section we show how to overcome this failure of long term forecasting for Koopman operator regression (KOR) estimators based on \eqref{eq:rkhs_flow} using an \textit{estimator agnostic paradigm} that consists of three steps: 
\begin{itemize}
    \item[\textbf{(D)}] Compute the empirical KME of the invariant distribution and remove it from the KME of the initial condition (deflate);
    \item[\textbf{(L)}] Compute an estimator from data via empirical risk minimization using centered features (learn); 
    \item[\textbf{(I)\,}] Evolve the initial KME using the deflated estimator and add to it the empirical KME of the invariant measure (inflate).
\end{itemize}
We proceed to explain each of these steps in turn.

{\bf Deflate~} The first step is a classical idea in the field of numerical methods for eigenvalue problems, see, e.g.,  \cite{saad2011numerical}. In our context, recalling \eqref{eq:specrad}, it consists of removing (\textit{deflating}) the known eigenpair $(1,\cf)$ of the Koopman operator $\Koop$, in order to better estimate the unknown ones. Since the leading Koopman eigenvalue $\lambda_1(\Koop)=1$ is simple, its corresponding spectral projector is $\cf\otimes\cf$. Hence, we obtain the \textit{deflated} operator 
\begin{equation}\label{eq:deflated_koop}
\cKoop:=\Koop-\cf\otimes\cf = \Koop\J = \J \Koop, 
\end{equation}
where $\J:=I-\cf\otimes\cf$ is the orthogonal projector onto the orthogonal complement of the subspace of constant functions. Then, applying $\J$ to the flow \eqref{eq:density_flow} and, noticing that $\scalarp{\cf,\dt{t}}_{\Lii}=1$, $t\in\N$, we obtain that 
\begin{equation}\label{eq:deflated_flow}
\dt{t} - \cf =  (\adjcKoop)^t(\dt{0}-\cf),\;t\in\T,
\end{equation}
{which, since $\srad(\adjcKoop)=\srad(\cKoop)<1$, defines an \textit{asymptotically stable linear dynamical system} in shifted variables. Hence, $\dt{t} -  \cf\to 0$, that is, $\mt{t}\to \im$, when $t\to\infty$.}

{\bf Learn~} To show how one can learn the deflated operator, we start by applying $\TS^*$ to \eqref{eq:deflated_flow}, which using $\TS^*\cf = \kme{\im}$, gives
\begin{equation}\label{eq:deflated_flow_error}
\kme{\mt{t}}\!-\!\kme{\im} \!=\! \TS^*(\adjcKoop)^t(\dt{0}\!-\!\cf)\approx\!  (\cEEstim^*)^{t}\,(\kme{\mt{0}}\!-\!\kme{\im}).
\end{equation}
Thus the estimator $\cEEstim\colon\RKHS\to\RKHS$ is learned so that $\TS^*\J\adjKoop \approx \cEEstim^* \TS^*\J$, which denoting $\cTS:=\J\TS$, can be written as 
$\cTS^* \cKoop^* \approx \cEEstim^* \cTS^*$. 
In this way we can write the estimation error of the deflated operator as
\begin{equation}
\label{eq:centeredExcessRisk}
    \cerror(\cEEstim):=\Vert \cKoop\cTS-\cTS \cEEstim\Vert
\end{equation}

and formulate the problem of learning $\cKoop$ as the minimization of the empirical risk functional
\begin{equation}\label{eq:empirical_risk_deflated}
{\cERisk}(\cEEstim) \!:=\! \Vert \cEZ- \cES \cEEstim\Vert^2_{\rm HS},
\end{equation}
where we have defined the empirical operators $\cES\,{=}\,\Jn\ES$ and $\cEZ\,{=}\,\Jn\EZ$, with $\Jn$ the $n \times n$ orthogonal projection matrix $\Jn \,{=}\, I \,{-}\, \cvn\cvn^\top$, and the vector $\cvn=(1/\sqrt{n})[1,\,1\,\ldots,\,1]^\top \in\R^n$.

Next, introducing two empirical estimators of the invariant measure $\im$ associated to the input and the output points,
\begin{equation}\label{eq:emp_im}
\eimx:= \tfrac{1}{n}\textstyle{\sum_{i\in[n]}}\delta_{x_i}\;\text{ and }\eimy:=\tfrac{1}{n}\textstyle{\sum_{i\in[n]}}\delta_{y_i},    
\end{equation}
and observing that for the $j$-th standard basis vector $e_j\in\R^n$, we have that 
$$
\cES^*e_j = \ES^*(e_j\,{-}\, \tfrac{1}{\sqrt{n}}\cvn)\,{=}\, \tfrac{1}{\sqrt{n}} (\kme{x_j}\,{-}\,\kme{\eimx})
$$
and similarly that $\cEZ^*e_j = \tfrac{1}{\sqrt{n}} (\kme{y_j}\!-\!\kme{\eimy})$. Using these, we can rewrite the empirical risk \eqref{eq:empirical_risk_deflated} as
\begin{equation}\label{eq:empirical_risk_centered}
    {\cERisk}(\cEEstim) \!:=\! \tfrac{1}{n}\textstyle{\sum_{i \in[n]}} \Vert{(\fm{y_{i}}\! -\! \kme{\eimy})- \cEEstim^{*}(\fm{x_i}\!-\!\kme{\eimx})}\Vert^{2}.
\end{equation}

This shows an elegant connection between the two learning problems. Namely, any estimator $\EEstim$ of the Koopman operator $\Koop$ can be transformed into an estimator $\cEEstim$ of the deflated Koopman operator $\cKoop$ simply by \textit{empirically centering the feature map} of the kernel. Moreover, as we show in the following section, the deflated version $\cEEstim$ of an estimator $\EEstim$ can readily be statistically studied by replacing Koopman operator $\Koop$, injection $\TS$ and sampling operators $\ES$ and $\EZ$ by the projected ones $\cKoop$, $\cTS$, $\cES$ and $\cEZ$, respectively. 

{\bf Inflate~} Finally, we use an estimator $\cEEstim$ to obtain the empirical estimates of the flow $(\mt{t})_{t \in \N}$. For this purpose, we need to put back (\textit{inflate}) the leading eigenpair that we have removed during the deflate step. Recalling \eqref{eq:deflated_flow_error}, this is equivalent to adding the KME of the invariant measure $\im$ after the evolution. While we do not know $\im$, we have data sampled from it. So, if instead of $\kme{\im}$ we inflate $\kme{\eimy}\!=\!\EZ^*\cvn$, obtaining 
\begin{equation}\label{eq:deflated_flow_rkhs}
\kme{\emt{t}}=\kme{\eimy}\! +\! (\cEEstim^*)^{t}\,(\kme{\mt{0}}\!-\!\kme{\eimy}).
\end{equation}
We consider estimators obtained by minimizing the empirical risk \eqref{eq:empirical_risk_centered}. Consequently, the estimators are of the form
$\cEEstim = \cES^* W \cEZ$ for some $W\in\R^{n\times n}$. Then, minding that in practice one usually has an empirical estimate $\emt{0}=n_0^{-1}\sum_{i\in[n_0]} \delta_{z_i}$ instead of the initial true measure $\mt{0}$, we obtain that $\emt{t}=\sum_{i\in[n]} w_{t,i}\delta_{y_i}$ where the coefficient vector $w_t = (w_{t,i})_{i=1}^n$ is given by 
\begin{equation}\label{eq:evoleved_weigths}
w_t\! =\!\tfrac{1}{\sqrt{n}}\left(\cvn\! +\! (\Jn W^*\Jn \Kxy)^{t\!-\!1} w_0\right)\!,
\end{equation}
for
\begin{equation}\label{eq:init_weights}
    w_0= \Jn W^*\Jn(\Kxz\cv{n_0}\!-\!\Kxy\cvn)
\end{equation}
and kernel Gram matrices 
\[
\Kxy{:=}\tfrac{1}{n}[k(x_i,\!y_j)]_{i,j\in[n]} \text{ and }
\Kxz{:=} \tfrac{1}{\sqrt{n_0 n}}[k(x_i,\!z_j)]_{\substack{i\in[n] \\ j\in[n_0]}}.
\]
An outstanding property of such an estimator is the \textit{preservation of the probability mass} along all trajectory, since it holds that $\sum_{j\in[n]}w_{t,j}\! =\! 1$ due to the properties of the projector $\Jn$.

In summary,
the DLI paradigm to forecast state distributions can be implemented via the following few steps: 
i) Compute the \textit{centered} kernel Gram matrices,
ii) use a kernel estimator of choice to minimize~\eqref{eq:empirical_risk_centered}, obtaining the matrix $W$, and iii) use the matrix $W$ to compute the evolved weights via ~\eqref{eq:evoleved_weigths}-\eqref{eq:init_weights}. 
Note that the computational complexity added by DLI is modest as it only amounts to centering the kernels with a cost of $O(n^2)$. Every other step, indeed, incurs the same computational complexity as its non-DLI counterpart.

\section{UNIFORM FORECASTING BOUNDS}
\label{sec:theory}
In this section we prove that the DLI paradigm transforms Koopman estimators which are only one step ahead consistent, into estimators that achieve uniform consistency over time according to the MMD error. 

\begin{theorem}
\label{th:mainforecastingdistributions}
Let $\mu_0\ll\im$, assume that $\density_0=d\mt{0}/d\im\in\Lii$, and let $\emt{0}$ be an empirical estimate of $\mt{0}$. If $\dnorm{\mt{0}- \emt{0}} \vee \dnorm{\im - \eimy} \leq \varepsilon$ for some $\varepsilon>0$, then for every $\cEEstim$ estimator of $\cKoop$ and 
every $t \ge 1$:
\begin{equation*}
\dnorm{\emt{t}-\mt{t}}\le\pcon(\cEEstim)\left(\cerror(\cEEstim)\,\Vert\dt{0}\Vert \, 
\sumcon(\cKoop) + \varepsilon\right)+\varepsilon,
\end{equation*}
where we recall that $\cerror(\cEEstim)$ is given in \eqref{eq:centeredExcessRisk}, while
\begin{equation}\label{eq:powercons}
\sumcon(\cKoop):=\sum_{t=0}^{\infty}\Vert \cKoop^t\Vert \quad\text{ and }\quad
\pcon(\cEEstim):=\sup_{t\in\N_0}\Vert \cEEstim^t\Vert.
\end{equation}
\end{theorem}
\begin{proof}%
First, observe that \eqref{eq:deflated_flow} and \eqref{eq:deflated_flow_error}, imply, by adding and subtracting $\kme{\emt{0}}$, that the prediction error $\dnorm{\emt{t} - \mt{t}}$ can be upper bounded by $\norm{\cKoop^t\cTS \!-\! \cTS \cEEstim^t}\norm{\dt{0}\!-\!\cf}+\norm{\cEEstim^t}\dnorm{\emt{0}\!-\!\emt{0}} + \dnorm{\eimy\!-\!\im}$.
So, since after some algebra, 
\[
\cKoop^t\cTS - \cTS \cEEstim^t = \textstyle{\sum_{k=0}^{t\!-\!1}}\,\cKoop^{k} (\cKoop\cTS - \cTS \cEEstim) \cEstim^{t-1-k},
\] 
applying the norm %
we obtain the result. 
\end{proof}
{The above theorem introduces in \eqref{eq:powercons} important quantities in the study of asymptotically stable linear dynamical systems in a Hilbert space, that is systems governed by a bounded operator $A$ such that $\srad(A)<1$. For such systems we have that $\lim_{t\to\infty}\norm{A^t}=0$, but, depending on the \textit{normality} of the operator, the convergence might not be monotone. Namely, if $A$ is normal, that is $AA^*=A^*A$, then $\norm{A^t} = [\srad(A)]^t$ and, consequently, $\pcon(A)=1$ and $\sumcon(A)=1/(1\!-\!\srad(A))$. Moreover, in this case $1/\sumcon(A)$ coincides with the \textit{distance to instability} of $A$ 
\begin{equation}\label{eq:d2i}
\dtoins(A):=\sup_{z\in\C,\,\abs{z}\geq1}\norm{(A\!-\! zI)^{-1}}^{-1}
\end{equation}
that measures the distance of the operator's spectra to the unit circle relative to its sensitivity to perturbations, which for normal operators equals $1\!-\!\srad(A)$, see e.g.~\cite{TrefethenEmbree2020}.} 

{On the other hand, for nonnormal operator $A$, the sequence $(\norm{A^t})_{t\in\N_0}$ may exhibit a \textit{transient growth} before converging to zero can, see e.g.~\cite{ELFALLAH2002135}, that can be estimated by $\kreiss(A)\leq \pcon(A)\leq (e/2) [\kreiss(A)]^2$, where $\kreiss(A)$ is the \textit{the Kreiss} constant of $A$ defined as 
\begin{equation}\label{eq:kreiss}
\kreiss(A):=\sup_{z\in\C,\,\abs{z}>1}(\abs{z}\!-\!1)\norm{(A\!-\! zI)^{-1}}\geq1.
\end{equation} 
For highly nonnormal operators $\kreiss(A)\gg 1$, indicating a large transient growth, which is also related to much smaller distance to instability $\dtoins(A)\ll 1\!-\!\srad(A)$, and larger cumulative effect $\sumcon(A) \gg 1/(1\!-\!\srad(A))$. Nevertheless, the latter quantity always remains bounded, since due to $\limsup_{t\rightarrow \infty}\norm{A^t}^{1/t} =  \rho(A)<1$, there exists the smallest integer $\ell$ such that $\norm{A^\ell}<1$, and, consequently, $\sumcon(A)\leq \tfrac{1}{1\!-\!\norm{A^\ell}}\tfrac{\norm{A}^\ell\!-\!1}{\norm{A}\!-\!1}<\infty$.}

While due to  $\srad(\cKoop)< 1$ we have that $\sumcon(\cKoop)$ is bounded, in the following result, the proof of which can be found in App.~\ref{app:proof_main_results}, we show how one can also bound the transient growth of an empirical estimator by proving its concentration in the RKHS operator norm.
\begin{restatable}{lemma}{powerbound}
\label{lm:powerbound}
If $\cKoop\cTS = \cTS \cHKoop$ holds for a compact operator $\cHKoop\colon\RKHS\to\RKHS$, then $\srad(\cHKoop)<1$. Consequently $1\leq\kreiss(\cHKoop)<\infty$, $\dtoins(\cHKoop)>0$ and for every $\cEEstim$ such that $\norm{\cHKoop\!-\!\cEEstim} < \dtoins(\cHKoop)$ it holds that 
\begin{equation*}
\frac{\kreiss(\cHKoop) \dtoins(\cHKoop)}{\dtoins(\cHKoop)\!+\!\norm{\cHKoop\!-\!\cEEstim} } \leq\pcon(\cEEstim)\leq \frac{e}{2} \left[\frac{\kreiss(\cHKoop) \dtoins(\cHKoop)}{\dtoins(\cHKoop)\!-\!\norm{\cHKoop\!-\!\cEEstim} }\right]^2.
\end{equation*}
\end{restatable}
In view of Thm.~\ref{th:mainforecastingdistributions} and Lem.~\ref{lm:powerbound}, for any KOR estimator $\cEEstim$ such that $\cerror(\cEEstim)$ is small and $\pcon(\cEEstim)$ is bounded, we can guarantee that the MMD distance between the forecasted marginal distribution $\emt{t}$ and $\mt{t}$ is small for any $t\geq 1$, that is uniformly along the whole trajectory of future states. %

We present below our analysis for the centered KRR and defer the analysis for PCR and RRR to App. \ref{app:proof_main_results}. 
In the DLI paradigm, we only need to replace $\ECx$ and $\ECxy$ by their centered version $\cECx$ and $\cECxy$ respectively to define the deflated versions
$\cERKoop,\cEPCR,\cERRR$.

We now state our main assumptions on the centered operators. We use the symbol ‘‘$*$" to distinguish assumptions on the centered operators $\cCx$, $\cCxy$ from their standard (uncentered) operators $\Cx$ and $\Cxy$.

\begin{enumerate}[label={\rm \textbf{(BK)}},leftmargin=7ex,nolistsep]
\item
\emph{Boundedness.} \label{eq:BK} There exists $\bcon\,{>}\,0$ such that %
$\displaystyle{\esssup_{x\sim\im}}\, k(x,x)\leq \bcon$. 
\end{enumerate}
\begin{enumerate}[label={\rm \textbf{(RC*)}},leftmargin=7ex,nolistsep]
\item \emph{Regularity condition.} \label{eq:RC2} For some $\rpar\in[1,2]$ there exists $\rcon>0$ such that
$\cCxy \cCxy^* \preceq \rcon^2 \cCx^{1+\alpha}$.
\end{enumerate}
\begin{enumerate}[label={\rm \textbf{(SD*)}},leftmargin=7ex,nolistsep]
\item
\emph{Spectral Decay.} \label{eq:SD2}There exists $\spar\,{\in}\,(0,1]$ and a constant $\scon\,{>}\,0$ so that
$\eval_j(\cCx)\,{\leq}\,\scon\,j^{-1/\spar}$, $\forall j\in \N$.
\end{enumerate} 
Assumptions \ref{eq:BK} and \ref{eq:SD} are taken from the works \cite{Fischer2020,Li2022} on kernel mean embeddings. Assumption \ref{eq:RC} was introduced in \cite{Kostic2023b} to study KOR estimators. %

First, note that $\cCx\! = \!\cTS^*\cTS \!=\!   \TS^*\J\TS \!\preceq\! \TS^*\TS\!=\!\Cx$, and $\range(\Koop\TS)\!\subseteq\!\range(\TS)$ implies $\range(\cKoop\cTS)\!\subseteq\!\range(\cTS)$. Thus,  we have that conditions \ref{eq:RC2} and \ref{eq:SD2} are weaker than conditions \ref{eq:RC} and \ref{eq:SD}, respectively.
That is, if $\Cx$ and $\Cxy$ satisfy \ref{eq:RC} and \ref{eq:SD}, then the centered objects $\cCx$ and $\cCxy$ satisfy \ref{eq:RC2} and \ref{eq:SD2} with possibly smaller $\beta$.

{
The 
proof argument in \cite{Kostic2023b} 
relies at its core on the SVD decomposition of the injection operator $\TS$.~Hence, since the main difference between centered and uncentered operators arises from using the projected injection operator $\cTS$ instead of $\TS$, their analysis could be extended in an elegant way to control $\cerror(\cEEstim)$.} %
To this end, we only need to extend some bounds on whitened features to the centered case. 
\begin{restatable}{proposition}{propWhithened}\label{prop:whithened}  
Let \ref{eq:BK} and \ref{eq:SD2} hold for some $\spar\in(0,1]$, and let 
$\xi(x): =\cCreg^{-1/2}[\fm{x}-\kme{\im}]$.
Then there exist $\epar\in[\spar,1]$ and $\econ, c_{\spar}\in(0,\infty)$  such that for $\reg{>}0$
\begin{equation}\label{eq:whitened_norms}
\norm{\xi}^2 \leq c_{\spar} \reg^{-\spar} \;\text{ and }\; \norm{\xi}_\infty^2 \leq c_{\epar} \reg^{-\epar}.
\end{equation}
\end{restatable}
Therefore, variance control of estimators from \cite{Kostic2023b}  can be readily applied which leads to the learning rates for centered KRR, RRR and PCR estimators.

We now specify Thm.~\ref{th:mainforecastingdistributions} for the KRR estimator.  %
A similar result holds for the PCR and RRR estimators with $\rpar\in [1,2]$ (see App. \ref{app:proof_main_results} for more details). %
\begin{theorem}\label{prop:E_t_RRR}
Assume the operator $\cKoop$ satisfies $\rho(\cKoop) <1$.
Let \ref{eq:SD2} and \ref{eq:RC2} hold for some $\spar\in(0,1]$ and $\rpar\in(1,2]$, respectively. In addition, let $\cl(\range(\TS))=\Lii$ and \ref{eq:BK} be satisfied. Let
\begin{equation}\label{eq:opt_reg}
    \reg\asymp n^{-\frac{1}{\rpar+\spar}}\,\text{ and }\,\rate^\star_n:= n^{-\frac{\rpar}{2(\rpar+\spar)}}.
\end{equation}
Let $\delta\in(0,1)$. Then the forecasted distributions \eqref{eq:deflated_flow} based on KRR %
satisfy with probability at least $1-\delta$, %
\begin{equation*}
\label{eq:error_bound_dist_krr}
\dnorm{\emt{t}-\mt{t}}\!\le C \! %
\left( \rate_n^\star\,\ln(\delta^{-1})+ \sqrt{\frac{\ln \delta^{-1}}{n_{0} \wedge n}}\right),\; \forall t\geq 1,
\end{equation*}
where the constant $C$ may depend only on $a,b,\bcon$. 
\end{theorem}
\begin{proof}[Proof Sketch]
Recall that the centered population {({\bf KRR})} model is defined as $\cRKoop = \cCx_\reg^{-1} \cCxy$ where $\cCx_{\reg} := \cCx + \reg\Id_{\RKHS}$, while the empirical estimator is $\cERKoop = \cECx_\reg^{-1} \cECxy$.  Now, in view of Thm.~\eqref{th:mainforecastingdistributions} and Lem.~\ref{lm:powerbound}, it suffices to prove that $\norm{\cHKoop\!-\!\cERKoop}\leq \dtoins(\cHKoop)/2$ w.h.p. and derive the learning rate for $\cerror(\cERKoop)$. First, since, $\norm{\cHKoop\! - \!\cERKoop}\leq \norm{\cHKoop \!-\! \cRKoop} \!+\! \norm{\cRKoop\! -\! \cERKoop}
$, Lem.~\ref{lm:Gh-Ggamma-KRR} in App. \ref{app:proof_main_results} gives that $ \norm{\cHKoop - \cRKoop}\leq  \rcon \reg^{(\rpar-1)/2}$. Next, using \cite[Proposition 16]{Kostic2023b}, for the KRR estimator we obtain that 
$\PP\left\lbrace\norm{\cERKoop\!-\! \cRKoop}\leq C
\frac{\ln(2\delta^{-1})}{\sqrt{n\reg^{\spar+1}}}\right\rbrace\! \geq\! 1\!-\!\delta$. Thus, provided that $\rpar>1$ and $n$ is taken large enough, with our choice of $\reg$, we can guarantee that $\PP\left\lbrace\norm{\cHKoop\!-\!\cERKoop}\leq \dtoins(\cHKoop)/2\right\rbrace\geq 1-\delta$. Next, by applying Propositions 5 and 16 from \cite{Kostic2023b}, we obtain that $\cerror(\cERKoop)\leq C
\rate_n^\star\,\ln(\delta^{-1})$.
Finally, \cite[Lem.~1]{briol2019statistical} guarantees with probability at least $1-\delta$ that $\dnorm{\mt{0}- \emt{0}} \le \epsilon$ with $\epsilon = \sqrt{\frac{2}{n_{0}}\bcon}\left( 1 + \sqrt{\log(\delta^{-1})}\right)$. A similar bound also holds for $\dnorm{\im - \eimy}$ with $n_{0}$ replaced by $n$. Finally, an union bound combining the previous results with Thm.~\ref{th:mainforecastingdistributions} yields the result
with probability at least $1-4\delta$. Up to a rescaling of the constants, we can replace $1-4\delta$ by $1-\delta$.%
\end{proof}
According to Thm.~\ref{prop:E_t_RRR}, the DLI paradigm enables learning operators that can reliably forecast future state distributions, uniformly over time. 
Notice that in practice, we can easily sample from $\mt{0}$, so we can make $n_0 \geq n$ and then the dominating term in \eqref{eq:opt_reg} is $\rate^\star_n$, which depends only on the spectral properties of the kernel embedding and the transfer operator.

\section{EXPERIMENTS}
\label{sec:exp}

In this section, we illustrate the practical impact of the DLI paradigm through two specific applications, where we trained the most advanced~\cite{Kostic2023b} low-rank Koopman estimator currently available, that is Reduced Rank Regression (RRR), and compared it to its DLI-enhanced counterpart. In both examples, the DLI paradigm boosts the performance of the bare RRR model in distribution forecasting. 

In both models a Gaussian kernel has been used and the hyperparameters have been optimized via grid search on a validation split.
\paragraph{Ornstein-Uhlenbeck model.}
We study the uniformly sampled Ornstein-Uhlenbeck process $dX_{t} = -\theta X_{t}dt + \sigma dW_{t}$, where 
$\theta, \sigma > 0$ and $W_{t}$ is a Wiener process. Integrating the stochastic differential equation shows that the probability flow of $X_{t}$ (given that $X_{0} = x$) is
    $\mt{t} = \mathcal{N}\big(xe^{-\theta t}, \frac{\sigma^{2}}{2\theta}(1 - e^{-2\theta t}) \big)$,
yielding an invariant measure $\im = \mathcal{N}(0, \sigma^{2}/2\theta)$. We investigate the equilibration of an Ornstein-Uhlenbeck process with initial condition $X_{0}$ drawn from a Gaussian mixture with means $\{ -2, 2\}$ and variances $\{0.04, 0.04\}$, respectively. In Fig.~\ref{fig:rel_MMD} we report the predicted probability flow, as well as the {\em relative} MMD $\dnorm{\emt{t}-\mt{t}}/\dnorm{\mt{t}}$ attained by DLI and uncentered estimators. 
Note that the DLI paradigm leads to a consistent improvement in forecasting performance throughout the entire trajectory.

In a recent work, \cite{Kostic2023b} studied Koopman operator learning via RKHS, highlighting that adoption of a universal kernel, although powerful in many aspects, 
can introduce metric distortions. This, in turn, gives rise to a non-normal Koopman operator within the RKHS utilized 
by the learning algorithm. Interestingly, this non-normality triggers a transient growth regime within the system's dynamics, resulting in short-term instability before eventually converging towards a stable, long-term behavior~\cite{TrefethenEmbree2020}.  
This intriguing phenomenon was empirically observed in Fig.~\ref{fig:rel_MMD}, where a Gaussian kernel is used. While the true Koopman operator is self-adjoint, its representation within the RKHS induced by the RBF kernel is non-normal. This resulted in a hump in MMD error for both the RRR model and its DLI version during the short-term forecasting, as anticipated by our theory.

\paragraph{US Mortgage rates.}
\begin{figure}
  \centering
  \includegraphics[width=0.95\columnwidth]{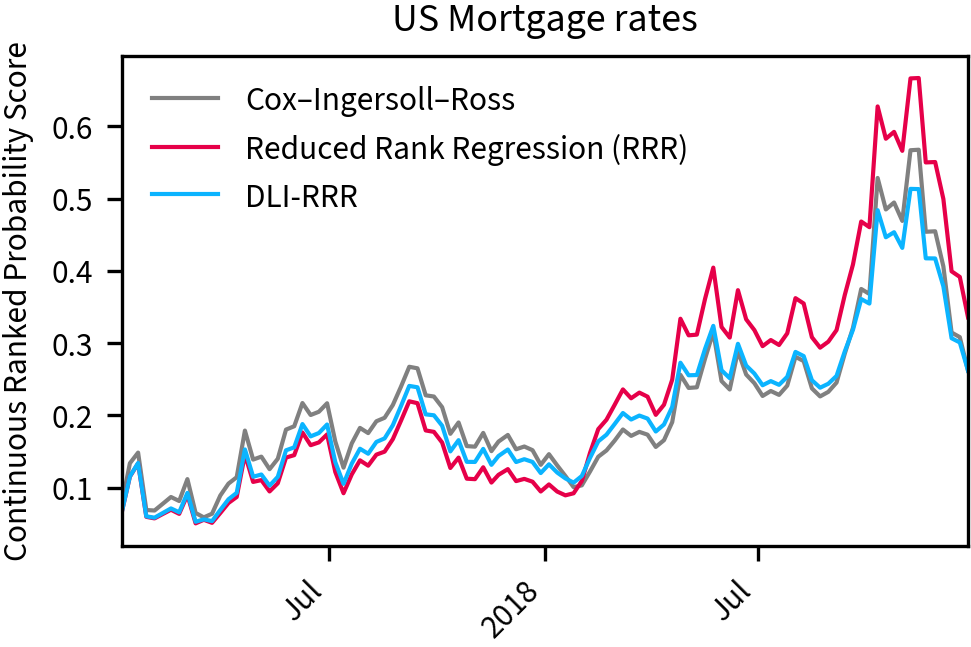}
  \caption{ Continuous ranked probability score attained by different esimators of the 30-year US mortgage rates. Weekly observations.}
  \label{fig:crps}
\end{figure}

In our next experiment, we study weekly historical data for the 30-Year fixed-rate mortgage averages in US. Interest rates are commonly modeled by the Vasicek~\cite{Vasicek1977} model, an Ornstein-Uhlenbeck process with non-zero mean, or the more advanced Cox–Ingersoll–Ross~\cite{Cox1985} model (CIR), which are both known to have an invariant distribution. We compare a least-squares calibration of the latter against a Gaussian-kernel RRR and its DLI enhancement, showing improved performance of the DLI estimator. Each model is evaluated according to the {\em continuous ranked probability score}\footnote{Despite its name, it is actually a loss function and not a score function.}~\cite{Hersbach2000} (CRPS), which measures the difference between the predicted and occurred cumulative distributions. Indeed, if at time $t$ a mortgage rate $x_{t}$ has been observed, and by letting $\hat{F}_{t}(x):= {{\rm cdf}}(\emt{t})(x)$, the CRPS is defined as as $${\rm CRPS}(\emt{t}) := \int_{-\infty}^{x_{t}} \hat{F}_{t}^{2}(x)dx + \int_{x_{t}}^{\infty} (1 - \hat{F}_{t}(x))^{2}dx,$$
that is the $L^{2}$-error between the cumulative distribution functions of $\emt{t}$ and $\delta_{x_{t}}$. In Fig.~\ref{fig:crps} we plot the CRPS for the predicted mortgage rates of the years 2017 and 2018, using as initial condition the rates for the last week of 2016. Each model has been trained (or calibrated) using data from Jan 2009 to Dec 2016. The CIR model relies on prior knowledge of the equation of evolution to predict future mortgage rates. Consequently, it involves solving a parametric problem to estimate the two unknown parameters of this model. In contrast, the DLI approach is entirely data-driven and outperforms the CIR model in long-term forecasting, as demonstrated in Fig. \ref{fig:crps}, and on average across the entire trajectory, as shown in Table \ref{table:crps-comparison}.

\begin{table}[t]
\begin{center}
\begin{tabular}{l|c}
\toprule
Model & Average CRPS  \\
\midrule
Cox--Ingersoll--Ross & 0.213 $\pm$  0.110\\
Reduced Rank Regression & 0.228 $\pm$  0.153 \\
\textbf{DLI-RRR} & \textbf{0.203} $\pm$ \textbf{0.106} \\
\bottomrule
\end{tabular}
\caption{Average continuous ranked probability score of the estimators in Fig.~\ref{fig:crps} for the years 2017 and 2018.}
\label{table:crps-comparison}
\end{center}
\end{table}

\section{CONCLUSIONS}
\label{sec:conclusion}
In this paper, we have studied data-driven approaches for the long-term  forecasting of ergodic discrete dynamical systems. These systems, which may be either deterministic or stochastic, are fully represented by the associated Koopman or transfer operator. We focused on the problem of predicting the flow of distributions of the state of the system from an initial distribution. Motivated by the observation that mainstream KOR estimators may fail at this task, we presented a conceptually simple and statistically principled approach which solves the above problem. This approach transforms the dynamical system into a linear deterministic asymptotically stable one. Properties of the deflated operator, determined by its possible nonnormality, play an important role in the bound. A compelling feature of our method is its agnostic nature towards the KOR estimator. Moreover, it offers the advantages of low computational complexity and seamless integration with standard kernel scaling methods. In the future it would be interesting to extend our method and analysis to continuous-time systems as well as 
tackle non-stationary processes. 

\bibliographystyle{plain}
\bibliography{bibliography.bib}

\newpage
\onecolumn

\appendix

\begin{center}
{\bf \Large Supplementary material}    
\end{center}

The appendix is organised as follows:
\begin{itemize}
    \item Appendix \ref{app:background} provides a summary of important notations in Table \ref{tab:notation}, additional background on the Koopman operator framework and the detailed algorithm (Alg. \ref{alg:forecasting}) used to implement the DLI paradigm.
    \item Appendix \ref{app:proof_main_results} contains the complete proofs of the results in the main body of this paper.
    \item Appendix \ref{app:exp_details} contains additional details on the experiments we performed in the main body of the paper.
\end{itemize}

\section{Background}\label{app:background}

\begin{table}[!h]
\centering
\begin{tabular}{c|c}
\toprule
notation & meaning   \\ 
\midrule
$\mt{t}$ & law of the state of process at the time $t$ \\ \hline
$\dt{t}$ & density of the law of the state of process at the time $t$ w.r.t. invariant distribution\\ \hline
$k(\cdot,\cdot)$ & symmetric positive definite kernel function \\ \hline
$\phi(x)$ & canonical feature map associated to $x\in\X$ also denoted by $\kme{x}$\\ \hline
$\kme{\nu}$ & kernel mean embedding of the measure $\nu$ \\ \hline
$\Koop$ & Koopman operator \\\hline
$\TS$ & canonical injection of $\RKHS \hookrightarrow\Lii$ \\\hline
$\TZ$ & restriction of the Koopman operator to $\RKHS$ \\ \hline
$\cf$ &  function in $\Lii$ with the constant output 1 \\ \hline
$\J$ & projection onto $\one^\perp$ in $\Lii$\\ \hline
$\cKoop$ & deflated Koopman operator \\\hline
$\cTS$ & projected canonical injection of $\RKHS \hookrightarrow\Lii$ \\\hline
$\cTZ$ & restriction of the deflated Koopman operator to $\RKHS$ \\ \hline
$\Risk$ & true risk \\ \hline
$\cRisk$ & true centered risk \\ \hline
$\ERisk$ & empirical risk \\ \hline
$\cERisk$ & empirical centered risk \\ \hline
$\error$ & true error \\ \hline
$\cerror$ & true centered error \\ \hline
$\cvn$ &  normalized constant vector in $\R^n$ with all components equal to $1/\sqrt{n}$ \\ \hline
$\Jn$ & projection onto $\Span(\one_n)^\perp$ in $\R^n$
\\\hline
$\ES$ & sampling operator of the inputs \\ \hline
$\EZ$ & sampling operator of the outputs \\\hline
$\Cx$ & covariance operator \\ \hline
$\ECx$ & empirical covariance operator \\\hline
$\Cxy$ & cross-covariance operator \\ \hline
$\cCx$ & centered covariance operator \\ \hline
$\cECx$ & centered empirical covariance operator \\\hline
$\ECxy$ & empirical cross-covariance operator \\\hline
$\cCxy$ & centered cross-covariance operator \\ \hline
$\cECxy$ & centered empirical cross-covariance operator \\\hline
$\Kx$ & input kernel matrix \\ \hline
$\Ky$ &  output kernel Gramm matrix \\\hline
$\Kreg$ & regularized input kernel matrix \\ \hline
$\cKx$ & centered input kernel matrix \\ \hline
$\cKy$ &  centered output kernel Gramm matrix \\\hline
$\cKreg$ & regularized centered input kernel matrix \\ \hline
$\B{\RKHS}$ & the set of bounded operators on $\mathcal{H}$\\ \hline
$\HSr$ & the set of operators on $\mathcal{H}$ of a finite rank at most $r$\\\hline
$\HS{\RKHS}$ & the set of Hilbert-Schmidt (HS) operators on $\mathcal{H}$\\ \hline
$\Spec(\cdot)$ & spectrum of a bounded operator \\\hline
$\srad(\cdot)$ & spectral radius of a bounded operator \\\hline
$\dtoins(\cdot)$ & distance to instability of a bounded operator \\\hline
$\kreiss(\cdot)$ & Kreiss constant of a bounded operator \\\hline
\bottomrule
\end{tabular}
\caption{Summary of used notations.}\label{tab:notation}
\end{table}
\renewcommand{\arraystretch}{1}

\subsection{Markov Transfer Operators}
Let  $\mathbf{X} := \left\{X_{t} \colon t\in \N \right\}$ be a family of random variables with values in a measurable space $(\X, \sigalg)$, called state space. We call $\mathbf{X}$  a  {\em Markov chain} if $\PP\{ X_{t+1} \in B \,\vert\, X_{[t]} \} = \PP\{X_{t + 1} \in B \,\vert\, X_t \}$. 
Further, we call $\mathbf{X}$ {\em time-homogeneous} if there exists $\transitionkernel\colon \X \times \sigalg \to [0,1]$, called {\it transition kernel}, such that,  for every $(x, B) \in \X \times \sigalg$ and every $t \in \N$,
\[
\mathbb{P}\left\{X_{t + 1} \in B \middle| X_{t} = x \right\} = \transitionkernel (x,B).
\]

A large class of Markov chains consists of those endowed with \textit{invariant measure}, denoted as $\im$, that satisfies the equation $\pi(B) {=} \int_{\X} \pi(dx)p(x,B)$, $B\in\sigalg$, see e.g.~\cite{prato1996}. For such cases, we can consider the space of square-integrable functions on $\X$ relative to the measure $\im$, denoted as $\Lii$, and define the \textit{Markov transfer operator}, %
$\Koop \colon \Lii\to\Lii$ 
\begin{equation}\label{eq:Koopman_app}
	[\Koop f](x) := \int_{\X} p(x, dy)f(y) = \mathbb{E}\left[f(X_{t + 1}) \middle | X_{t} = x\right], \quad f\in\Lii,\,x\in\X.
\end{equation}
Since it easy to see that $\norm{\Koop f} \leq \norm{f}$, we conclude that $\norm{\Koop}\leq1$, i.e. the Markov transfer operator is a bounded linear operator. Moreover, recalling that $\cf(x)=1$ for $\im$-a.e. $x\in\X$, since $\Koop\cf = \cf$, we see that $1\leq \srad(\Koop)\leq\norm{\Koop} \leq 1$, i.e. $\srad(\Koop)=\norm{\Koop}=1$. 

\subsection{Koopman Mode Decomposition (KMD)}
In dynamical systems, $\Koop$ is known as the (stochastic) \textit{Koopman operator} on the space of  observables $\F=\Lii$. An essential characteristic of this operator is its linearity, which can be harnessed for the computation of a spectral decomposition. Indeed, in many situations, especially when dealing with compact Koopman operators, there exist complex scalars $\lambda_i\in\C$ and observables $\refun_i\in\Lii$ that satisfy the eigenvalue equation $\Koop\refun_i \hspace{.05truecm}{=}\hspace{.105truecm} \lambda_i\refun_i$. Leveraging the eigenvalue decomposition, the dynamical system can be decomposed
into superposition of simpler signals that can be used in different tasks such as system identification and control, see e.g. \cite{Brunton2022}.
More precisely, given an observable $f\in\Span\{\refun_i\,\vert\,i\in\N\}$ there exist corresponding scalars $\gamma_i^f\in\C$ known as Koopman modes of $f$, such that
\begin{equation}\label{eq:koopman_MD}
\Koop^t f(x) = \EE[f(X_t)\,\vert\, X_0 = x] =  \sum_{i\in\N}\lambda_i^t \gamma^f_i \refun_i(x), \quad x\in\X,\,t\in\N.
\end{equation}
This formula is known as \textit{Koopman Mode Decomposition} (KMD) \cite{Budisic2012,AM2017}. It decomposes the expected dynamics observed by $f$ into \textit{stationary} modes $\gamma_i^f$ that are combined with \textit{temporal changes} governed by eigenvalues $\lambda_i$ and \textit{spatial changes}  
governed by the eigenfunctions $\refun_i$. We notice however that the Koopman operator, in general, is not a normal compact operator, hence its eigenfunctions may not form a complete orthonormal basis of the space which makes learning KMD challenging.

\subsection{Kernel Based Learning of Koopman Operators}
\label{app:kooplearn}
In many practical situations, the Koopman operator $\Koop$ is unknown, but data from one or multiple system trajectories are available. A learning framework called Koopman Operator Regression (KOR) was introduced in~\cite{Kostic2022} to estimate the Koopman operator $\Koop$ on $\Lii$ using reproducing kernel Hilbert spaces (RKHS). More precisely, consider an RKHS denoted as $\RKHS$ with kernel $k:\X \times \X \rightarrow \mathbb{R}$ \cite{aron1950}. Let $\phi :\X \to \RKHS$ be an associated feature map such that $k(x,y) = \scalarp{\phi(x),\phi(y)}$ for all $x,y \in \X$. We assume that $k(x,x) \leq c_{\RKHS} < \infty$, $\pi$- almost surely.  This ensures that $\RKHS \subseteq \Lii$, and the injection operator $\TS \colon \RKHS \to \Lii$, defined as $(\TS f)(x)=f(x)$ for $x\in\X$, along with its adjoint  $\TS^* \colon \Lii \to \RKHS$ are well-defined Hilbert-Schmidt operators~\cite{caponnetto2007,Steinwart2008}. Then, the Koopman operator, when restricted to $\RKHS$, is given by 
\[
\TZ \colon \RKHS\to\Lii.
\]
Unlike $\Koop$, the operator $\TZ$ is Hilbert-Schmidt, which allows us to estimate $\TZ$ by minimizing the following {\em risk} 
\begin{equation}
\label{eq:true_risk}
\Risk(\Estim)= \EE_{x\sim \im} \sum_{i \in \N}  \EE \big[ (h_i(X_{t+1}) - (\Estim h_i)(X_t))^2\,\vert X_t  = x\big]
\end{equation}
over Hilbert-Schmidt operators $\Estim\in\HS{\RKHS}$, where $(h_i)_{i\in\N}$ is an orthonormal basis of $\RKHS$.
Moreover, we can write down a bias-variance decomposition of the risk  $\Risk(\Estim)=\IrRisk + \ExRisk(\Estim) $, where 
\begin{equation}
    \label{eq:ex_ir_risk}
\IrRisk=\hnorm{\TS}^2-\hnorm{\TZ}^2\geq0\;\text{ and }\; \ExRisk(\Estim)=\hnorm{\TZ-\TS\Estim}^2,
\end{equation}
are the irreducible risk (i.e. the variance term in the classical bias-variance decomposition) and the excess risk, respectively. This can be equivalently expressed in the terms of embedded dynamics in RKHS as:
\begin{equation}
    \label{eq:true_risk_cme}
\underbrace{ \EE_{(X,Y)} \norm{\phi(Y) - \Estim^*\phi(X)}^2}_{\Risk(\Estim)} =  \underbrace{\EE_{(X,Y)}\norm{\CME(X) - \phi(Y) }^2}_{\IrRisk} +  \underbrace{\EE_{X\sim \im} \norm{\CME(X) - \Estim^*\phi(X)}^2}_{\ExRisk(\Estim)},
\end{equation}
 where $(X,Y)$
is has the joint probability measure of two consecutive states of the Markov chain, and the regression function $\CME\colon\X\to\RKHS$ is defined as $\CME(x):=\EE[\phi(X_{t+1})\,\vert\,X_t = x]=\int_{\X}p(x,dy)\phi(y)$, $x\in \X$, and is known as the \textit{conditional mean embedding} (CME) of the conditional probability $p$ into $\RKHS$. It was also shown that using universal kernels one can approximate the restriction of Koopman arbitrary well, i.e. excess risk can be made arbitrarily small $\inf_{\Estim\in\HS{\RKHS}} \ExRisk(\Estim)= 0$.

Therefore, to develop estimators one can consider the problem of minimizing the Tikhonov regularized risk   
\begin{equation}\label{eq:KOR_reg}
\min_{\Estim \in \HS{\RKHS}} \Risk^\reg(\Estim){:=}\Risk(\Estim) + \reg\hnorm{\Estim}^2,
\end{equation}
where $\reg>0$. Denoting the covariance matrix as $\Cx := \TS^*\TS = \EE_{X\sim\im} \phi(X)\otimes \phi(X)$ and cross-covariance matrix $\Cxy : = \TS^*\TZ = \EE_{(X,Y)} \phi(X)\otimes \phi(Y)$, and regularized covariance as $\Creg:=\Cx+\reg\Id_\RKHS$, one easily shows that $\RKoop:=\Creg^{-1} \Cxy$ is the unique solution of \eqref{eq:KOR_reg} which is known as the Kernel Ridge Regression (KRR) estimator of $\Koop$. 

Low rank estimators of the Koopman operator have also been considered. Notably, Principal Component Regression (PCR) estimator given by $\SVDr{\Cx}^\dagger \Cxy$, where $\SVDr{\cdot}$ denotes the $r$-truncated SVD of the Hilbert-Schmidt operator. However, it is observed that both KRR and PCR estimators can fail in estimating well the leading Koopman eigenvalues \citep{Kostic2023b}. To mitigate this, Reduced Rank Regression (RRR) estimator has been introduced in \cite{Kostic2022} as the optimal one that solves \eqref{eq:KOR_reg} with an additional rank constraint by minimizing over the class of rank-$r$ HS operators $\HSr:=\{\Estim \in\HS{\RKHS}\,\vert\,\rank(\Estim)\leq r\}$, where $1\leq r <\infty$, i.e.
\begin{equation}\label{eq:KOR_RRR}
\Creg^{-1/2} \SVDr{ \Creg^{-1/2} \Cxy } = \argmin_{\Estim \in \HSr} \Risk^\reg(\Estim).
\end{equation}

Now, assuming that data $\Data = \{(x_i,y_i)\}_{i\in[n]}$ is collected, the estimators are typically obtained via the regularized empirical risk $\ERisk^\reg(\Estim){:=}\frac{1}{n}\sum_{i\in[n]} \norm{\phi(y_i) - \Estim^*\phi(x_i)}^2 +  \reg\hnorm{\Estim}^2$ minimization (RERM). Introducing  the sampling operators for data $\Data$ and RKHS $\RKHS$ by
\begin{align*}
    \ES \colon \RKHS \to \R^{n} \quad \text{ s.t. }  f \mapsto \tfrac{1}{\sqrt{n}}[ f(x_{i})]_{i \in[n]} & \quad \text{ and } &  \EZ \colon \RKHS \to \R^{n} \quad \text{ s.t. }  f \mapsto \tfrac{1}{\sqrt{n}}[ f(y_{i})]_{i \in[n]},
\end{align*}
 and their adjoints by
\begin{align*}
   \ES^* \colon \R^{n} \to \RKHS \quad \text{ s.t. } w \mapsto \tfrac{1}{\sqrt{n}}\sum_{i\in[n]}w_i\fH(x_i) & \quad \text{ and } & \EZ^* \colon \R^{n} \to \RKHS \quad \text{ s.t. } w \mapsto \tfrac{1}{\sqrt{n}}\sum_{i\in[n]}w_i\fG(y_i),
\end{align*}
we obtain  $\ERisk^\reg(\Estim){=}\hnorm{\EZ {-} \ES \Estim}^2 + \reg\hnorm{\Estim}^2$.

In the following we also use the empirical covariance operator defined as
\begin{equation}\label{eq:empirical_cov}
\ECx := \ES ^*\ES =\tfrac{1}{n} \sum_{i\in[n]}\phi(x_i)\otimes \phi(x_i),%
\end{equation}
and the empirical cross-covariance operator
\begin{equation}\label{eq:empirical_cross-cov}
\ECxy := \ES ^*\EZ = \tfrac{1}{n} \sum_{i\in[n]}\phi(x_i)\otimes \phi(y_i).
\end{equation}
Additionally, we let $\ECx_\reg : = \ECx+\reg \Id_{\RKHS}$ be the regularized empirical covariance. %
Then
we obtain the empirical estimators of the Koopman operator on an RKHS that correspond to the population ones:
empirical KRR estimator $\ERKoop:=\ECreg^{-1} \ECxy$, empirical PCR estimator $\SVDr{\ECx}^\dagger \ECxy$, and empirical RRR estimator $\ECreg^{-1/2} \SVDr{ \ECreg^{-1/2} \ECxy }$.

Noting that all of the empirical estimators above are of the form $ \EEstim = \ES U_r V_r^\top \EZ$,  where $U_r,V_r \in\R^{n\times r}$ and $r\in[n]$ which are computed using 
(normalized) kernel Gramm matrices $\Kx := \ES\ES^*=\frac{1}{n}[k(x_i,x_j)]_{i,j \in [n]}$ and $\Ky:= \EZ\EZ^* = \frac{1}{n}[k(y_i,y_j)]_{i,j \in [n]}$, see \cite{Kostic2022}. In the next section (especially in Theorem \ref{thm:estimators} and Alg. \ref{alg:forecasting}) we explain how to compute these estimators in practice within the proposed DLI framework.

\subsection{Deflate-Learn-Inflate Estimation of Koopman Operator}
\textbf{1) Deflation: projecting the Koopman operator} \\ \\
This step consists in projecting the Koopman operator onto the subspace of $\Lii$ orthogonal to $\cf$: the operator to learn is no longer $\Koop$ but $\Koop - \cf \otimes \cf$. 
Denote by $\J = \Id - \cf \otimes \cf$ the orthogonal projector onto $\Span(\cf)^\top$ in $\Lii$, and two operators $\J$ and $\Koop$ commute since $\one$ is a left and right singular function of $\Koop$. Hence, we have that $\Koop - \cf \otimes \cf = \Koop \J =  \J\Koop = \J\Koop\J$, which we denote by $\cKoop$, which implies that the restriction of the deflated Koopman operator to $\RKHS$, i.e. $\cKoop\TS$ can be written as $\cTZ = \J\Koop \TS$, where $\cTS:=\J\TS$ is projected injection operator. 

\begin{remark}
In the specific context of Koopman regression the process of deflation is equivalent to centering the feature map. Indeed deflation leads to learning the Koopman operator projected onto the $\Lii$-orthogonal subspace of the constant function, that is zero mean functions. On the sample level, this means substracting the empirical mean of the data set, a procedure which is therefore equivalent to centering the feature map. Consequently, the following procedure is motivated by the theory of centering feature maps. 
\end{remark}

\noindent \textbf{2) Learn: compute estimators for} $\bm{\Koop \J}$ \\

Learning $\Koop \J$ is the most interesting part of the Deflate-Learn-Inflate process, based on regression techniques. Previous literature on Koopman learning (\cite{Kostic2022} \cite{Kostic2023b} \cite{Li2022} \cite{Klus2019} \cite{klus2018data}) provides three popular estimators based on a statistics approach which we proceed to introduce. Since the regression problem is now learning the projected Koopman operator $\Koop \J$ we will hereby, analogously to  \eqref{eq:ex_ir_risk}, introduce the risk associated to this learning problem for a Hilbert-Schmidt operator $\cEstim\in\HS{\RKHS}$ as $\cRisk(\Estim)=\cIrRisk + \cExRisk(\cEstim) $, where 
\begin{equation}
    \label{eq:centered_ir_risk}
\cIrRisk:=\hnorm{\J\TS}^2-\hnorm{\J\TZ}^2 = \hnorm{\cTS}^2-\hnorm{\cKoop\cTS}^2 \geq0,
\end{equation}
is the centered irreducible risk (i.e. the variance term in the classical bias-variance decomposition) 
and 
\begin{equation}
    \label{eq:centered_ex_risk}
\cExRisk(\Estim):=\hnorm{\J\TZ-\J\TS\cEstim}^2= \hnorm{\cTZ-\cTS\cEstim}^2,
\end{equation}
is the centered excess risk. After some algebra, similarly as in \cite{Kostic2022}, centered risk can be equivalently expressed as:
\begin{equation}
    \label{eq:centered_true_risk}
\cRisk(\cEstim)= \EE_{(X,Y)} \norm{(\kme{Y}-\kme{\im}) - \cEstim^*(\kme{X}-\kme{\im})}^2.
\end{equation}

The following section explains how an estimator $\cTS \cEstim$ for the Hilbert-Schmidt operator $\cTZ$ is computed using regression algorithms. 
Given a dataset $\Data := (x_{i}, y_{i})_{i = 1}^{n}$ of sampled consecutive states, we introduce the empirical mean square error of an estimator $\cEstim$ for $\cTZ$: 
\begin{equation}
    \label{eq:empirical_risk_app}
    \cERisk (\cEstim) = \frac{1}{n}\sum_{i = 1}^{n} \left\Vert \left( \phi(y_{i})- \frac{1}{n} \sum_{i = 1}^{n} \phi(y_{i}) \right) - \cEstim^{*}\left(\phi(x_i) - \frac{1}{n} \sum_{i = 1}^{n} \phi(x_{i}) \right)\right\Vert^{2}_{\rm{HS}} 
\end{equation}
which is merely the empirical version of \ref{eq:centered_true_risk}. 
the excess risk associated to an estimator $\cEstim$ for $\cTZ$. Notice that here we are centering the feature map on the input and output space (by removing the mean). We also introduce the regularised risk which defines two popular estimators: 
\begin{equation}
    \label{eq:regularised_empirical_risk}
    \cERisk_{\reg} (\cEstim) = \cERisk (\cEstim) + \reg \norm{\cEstim}_{HS}^2
\end{equation}

The framework of statistical learning translates the approximation problem to a minimisation problem on the space of Hilbert-Schmidt operator acting on $\RKHS$. This formulation was not introduced in the setting of Koopman operator regression until \cite{Kostic2022} and was key to giving some statistical insight onto the different learning techniques. The paper gives various statistical properties of three supervised learning algorithms: Kernel Ridge Regression (KRR), Principal Component Regression (PCR) and Reduced Rank Regression (RRR). The KRR estimator minimises the empirical regularised risk \eqref{eq:regularised_empirical_risk} whilst the RRR algorithm minimises the same regularised risk under a fixed rank constraint. On the other hand the PCR estimator does not minimise the empirical risk but projects the output on the leading r eigenvectors of the covariance operator, that is the vectors responsible for the most variability of the inputs. 

\medskip
All of these estimators can be expressed in a unified form. Namely, 
denoting the normalized constant vector in $\R^n$ by $\one_n :=n^{-1/2}[1, \dots, 1]^T$ and the orthogonal projector to orthogonal complement by $\Jn := \Id_n - \one_n \otimes \one_n$, we can introduce the projected sampling operators $\cES:=\Jn \ES$ and $\cEZ:=\Jn \EZ$, which, , recalling that $\kme{\im}=\EE_{X\sim\im}\phi(X)$,  $\eimx=\sum_{i\in[n]} \phi(x_i)$ and $\eimy=\sum_{i\in[n]} \phi(y_i)$, are used to define empirical versions 
\[\cECx := \cES^*\cES = \sum_{i\in[n]} \left(\phi(x_i)- \kme{\eimx}\right) \otimes \left(\phi(x_i)- \kme{\eimx}\right)
\text{ and } 
\cECxy := \cES^*\cEZ = \sum_{i\in[n]} \left(\phi(x_i)- \kme{\eimx}\right) \otimes \left(\phi(y_i)- \kme{\eimy}\right),
\]
of the centered covariance matrix as $\cCx := \cTS^*\cTS = \EE_{X\sim\im} \left(\phi(X)- \kme{\im}\right) \otimes \left(\phi(X)- \kme{\im}\right)$ and centered cross-covariance matrix $\cCxy : = \cTS^*\cTZ = \EE_{(X,Y)}\left(\phi(X)- \kme{\im}\right)\otimes \left(\phi(Y)- \kme{\im}\right)$, respectively. Furthermore, we can introduce centered Kernel matrices associated to the input points:
\begin{equation}\label{eq:centered_kernels}
\cKx := \Jn\Kx\Jn = \Kx - (\Kx\cvn)\cvn^\top - \cvn (\Kx\cvn)^\top + (\cvn^\top\Kx\cvn)\cvn\cvn^\top
\end{equation} 
and, analogously, the one associated to the output points: $\cKy := \Jn\Ky \Jn $. 
Then, a unified form for centered empirical estimators is: $\cEstim = \cES^*  W \cEZ = \ES^* \Jn W \Jn \EZ $, where $W$ is a square matrix of size n (the number of samples), which we will refer to as the matrix form of the estimator from now on. The following theorem gives the expression of the matrix form of the estimators derived by the previously discussed algorithms. The proofs closely follow those presented in \cite{Kostic2022} for the estimators of the Koopman operator $\Koop$. The reader is encouraged to read this paper which introduces the empirical risk minimisation problem.

\begin{theorem}
\label{thm:estimators}
The deflated Koopman operator restricted to the RKHS $\RKHS$, i.e.  $\cTZ=\cKoop\TS$, can be empirically estimated by $\cEEstim = \cES^* W  \cEZ$, where $W\in\R^{n\times n}$ is determined as follows:
    \begin{itemize}
        \item[(i)]  The Kernel Ridge Regression (KRR) algorithm yields  $\cEEstim = \cERKoop:= \cECreg^{-1}\cECxy$ and $W = (\cKx + \gamma I_n)^{-1}$.
        \item[(ii)]The Principal Component Regression algorithm (PCR) yields  $\cEEstim=\cEPCR:=\SVDr{\cECx}^\dagger \cECxy$ and $W = U_r V_r^\top$, where $[\![ \cKx ]\!] _r = V_r \Sigma_r V_r^\top$ is the r-trunacted SVD of $\cKx$ and $U_r:= V_r \Sigma_r^{\dag}$. 
        \item[(iii)] The Reduced Rank Regression algorithm (RRR) yields $\cEEstim=\cERRR:=\ECreg^{-1/2} \SVDr{ \ECreg^{-1/2} \ECxy }$ and $W = U_r V_r^\top$, where $V_r:=\cKx U_r$ and $U_r = [u_1 \vert \dots \vert u_r] \in \mathbb{R}^{n\times r}$ is such that $(\sigma_i, u_i)$ are the solutions to the generalised eigenvalue problem : 
        $$ \cKy \cKx u_i = \sigma_i^2 (\cKx + \gamma I_n) u_i \ \text{normalised such that} \ u_i^\top \cKx ( \cKx + \gamma I_n) u_i = 1.$$ 
    \end{itemize}
\end{theorem}

\noindent \textbf{3) Inflation: preservation of the probability mass} \\
Finally, we use $\cEEstim$ to obtain the empirical estimates of the flow $(\mt{t})_{t \in \N}$ by putting back the leading eigenpair that we have removed during the deflate step. In view of \eqref{eq:deflated_flow_rkhs}, we now define
\[
\kme{\emt{t}}=\kme{\eimy}\! +\! (\cEEstim^*)^{t}\,(\kme{\emt{0}}\!-\!\kme{\eimy}).
\]
where $\emt{0}=n_0^{-1}\sum_{i\in[n_0]} \delta_{z_i}$  is the initial empirical measure. Recalling \eqref{eq:evoleved_weigths}-\eqref{eq:init_weights},  using Theorem \ref{thm:estimators} we obtain Algorithm  \ref{alg:forecasting} that results in the sequence of empirical measures $\emt{t}=\sum_{i\in[n]} w_{t,i}\delta_{y_i}$, supported on the output points $(y_i)_{i\in[n]}$, that, by construction satisfy $\emt{t}(\X)=1$ since  $\sum_{j\in[n]}w_{t,j}\! =\! \cvn^\top\cvn =1$ due to $\Jn\cvn=0$. 

\begin{algorithm}[h!]
\caption{Forecasting measures with KRR/PCR/RRR estimator via DLI framework} \label{alg:forecasting}
\begin{algorithmic}
\REQUIRE Dataset $\Data_n=(x_i,y_i)_{i\in[n]}$ from the process in stationary regime and samples $(z_i)_{i\in[n_0]}$ from some initial measure $\mt{0}$; hyperparameters $\reg>0$ and/or $r\in[n]$; forecasting horizon $T\in\N$.
\IF[\MP{\textit{KRR estimator}}]{$r=n$} 
\STATE Solve $(\cKx+\reg I)\tilde{w}_0=  \Jn(\Kxz\cv{n_0}-\Kxy\cvn)$ in $\widetilde{w}_0$ 
\STATE Update $\widetilde{w}_0 \leftarrow \Jn \widetilde{w}_0 $
\FOR{ $t=1,\ldots,T-1$}
\STATE Compute $w_t\leftarrow (\cvn+ \widetilde{w}_{t-1}) / \sqrt{n}$
\STATE Solve $(\cKx+\reg I)\widetilde{w}_{t} = \Jn \Kxy \tilde{w}_{t-1}$ in $\widetilde{w}_t$ 
\STATE Update $\widetilde{w}_t \leftarrow \Jn \widetilde{w}_t $
\ENDFOR
\STATE Compute $w_T\leftarrow (\cvn+ \widetilde{w}_{T-1}) / \sqrt{n}$
\ELSE[\MP{\textit{low rank estimators}}]
\IF[\MP{\textit{PCR estimator}}]{$\reg=0$}
\STATE Compute $U_r,V_r\in\R^{n\times r}$ using Theorem \ref{thm:estimators}(ii)
\ELSE[\MP{\textit{RRR estimator}}]
\STATE Compute $U_r,V_r\in\R^{n\times r}$ using Theorem \ref{thm:estimators}(iii)
\ENDIF
\STATE Update $U_r \leftarrow \Jn U_r $ and $V_r \leftarrow \Jn V_r $
\STATE Compute $\widetilde{w}_0\leftarrow U_r^\top( \Kxz\cv{n_0}-\Kxy\cvn)$
\STATE Compute $M \leftarrow U_r^\top\Kxy V_r$
\FOR{ $t=1,\ldots,T-1$}
\STATE Compute $w_t\leftarrow (\cvn+V_r^\top \widetilde{w}_{t-1}) / \sqrt{n}$
\STATE Update $\widetilde{w}_{t}\leftarrow M \widetilde{w}_{t-1}$
\ENDFOR
\STATE Compute $w_T\leftarrow (\cvn+V_r^\top \widetilde{w}_{T-1})/ \sqrt{n}$
\ENDIF
\RETURN Sequence of empirical measures $\emt{t}=\sum_{i\in[n]} w_{t,i}\delta_{y_i}$, $t\in[T]$.
\end{algorithmic}
\end{algorithm}

\section{Proofs of Main Results}
\label{app:proof_main_results}

\subsection{Main Assumptions}\label{app:assumptions}

The following assumptions were used in 
Sec.~\ref{sec:theory} to derive the learning bounds: 
\begin{enumerate}[label={\rm \textbf{(BK)}},leftmargin=7ex,nolistsep]
\item
\emph{Boundedness.} \label{eq:BK} There exists $\bcon\,{>}\,0$ such that $\displaystyle{\esssup_{x\sim\im}}\, k(x,x)\leq \bcon$,
i.e. $\phi\in L^\infty_\im(\X,\RKHS)$. 
\end{enumerate}
\begin{enumerate}[label={\rm \textbf{(RC)}},leftmargin=7ex,nolistsep]
\item \emph{Regularity condition.}  \label{eq:RC}For some $\rpar\in[1,2]$ there exists $\rcon>0$ such that
$\Cxy \Cxy^* \preceq \rcon^2 \Cx^{1+\alpha}$, with $\Cxy = \TS^* \TZ$.
\end{enumerate}
\begin{enumerate}[label={\rm \textbf{(RC*)}},leftmargin=7ex,nolistsep]
\item \emph{Regularity condition on the deflated operator} \label{eq:RC2} For some $\rpar\in[1,2]$ there exists $\rcon>0$ such that
$\cCxy \cCxy^* \preceq \rcon^2 \cCx^{1+\alpha}$.
\end{enumerate}
\begin{enumerate}[label={\rm \textbf{(SD)}},leftmargin=7ex,nolistsep]
\item
\emph{Spectral Decay.} \label{eq:SD}There exists $\spar\,{\in}\,(0,1]$ and a constant $\scon\,{>}\,0$ such that
$\eval_j(\Cx)\,{\leq}\,\scon\,j^{-1/\spar}$, for all $j\in J$.
\end{enumerate} 
\begin{enumerate}[label={\rm \textbf{(SD*)}},leftmargin=7ex,nolistsep]
\item
\emph{Spectral Decay of the centered operator.} \label{eq:SD2}There exists $\spar\,{\in}\,(0,1]$ and a constant $\scon\,{>}\,0$ such that
$\eval_j(\cCx)\,{\leq}\,\scon\,j^{-1/\spar}$, for all $j\in J$.
\end{enumerate} 

We start by observing that $\TS\in\HS{\RKHS,\Lii}$, and, hence $\cTS\in\HS{\RKHS,\Lii}$, too. Hence,  according to the spectral theorem for positive self-adjoint operators, has an SVD, i.e. there exists at most countable positive sequence $(\sigma_j)_{j\in N}$, where $N:=\{1,2,\ldots,\}\subseteq\N$, and ortho-normal systems $(\ell_j)_{j\in N}$ and $(h_j)_{j\in N}$ of $\cl(\range(\cTS))$ and $\Ker(\cTS)^\perp$, respectively, such that $\cTS h_j = \sigma_j \ell_j$ and $\cTS^* \ell_j = \sigma_j h_j$, $j\in N$.  Moreover, since $\cl(\range(\cTS))\subseteq\range(\J)$, we also have $\J\ell_j=\ell_j$, i.e. $\EE_{X\sim\im}[\ell_j(X)]=0$, $j\in N$.

Now, given $\rpar\geq 0$, let us define scaled injection operator $\cTSs{\rpar} \colon \RKHS \to \Lii$ as
\begin{equation}\label{eq:injection_scaled}
\cTSs{\rpar}:= \sum_{j\in N}\sigma_j^{\rpar}\ell_j\otimes h_j.
\end{equation}
Clearly, we have that $\cTS = \cTSs{1}$, while $\range{\cTSs{0}} = \cl(\range(\TS))$. Next, we equip $\range(\cTSs{\rpar})$ with a norm $\norm{\cdot}_\rpar$ to build an interpolation space. 
\[
\cRKHSs{\rpar}:=\left\{ f\in\range(\cTSs{\rpar})\;\vert\; \norm{f}_\rpar^2:= \sum_{j\in N}\sigma_j^{-2 \rpar} \scalarp{f,\ell_j}^2 <\infty \right\}.
\]

We remark that for $\rpar=1$ the space $\cRKHSs{\rpar}$ is just an RKHS $\RKHS$ seen as a subspace of $\range(\J)\subseteq\Lii$.  Moreover, we have the following injections
\[
 \cRKHSs{\rpar_1} \hookrightarrow \cRKHSs{1} \hookrightarrow \cRKHSs{\rpar_2}  \hookrightarrow \cRKHSs{0}  \hookrightarrow \range(\J)\subseteq\Lii,
\]
where $\rpar_1\geq1\geq\rpar_2\geq0$.

{\bf Regularity condition.} According to \cite[Theorem 2.2]{zabczyk2020}, the condition \ref{eq:RC2} is in fact equivalent to 
\[
\range(\cTZ) \subseteq \range(\cTSs{\rpar})\;\text{ and, hence, }\; \cTZ=\cTSs{\rpar}\cEstim_\RKHS^\rpar,\text{ where } \cEstim_\RKHS^\rpar:=\cTSs{\rpar}^\dagger \cCxy\in\B{\RKHS}.
\]

\begin{remark}[\textbf{Invariance of a RKHS}]\label{rem:strong_rc}
When $\rpar\geq1$, we necessarily have that $\range(\cTZ)\subseteq\range(\cTS)$, i.e. $\RKHS$ is $\im$-a.e. invariant under the conditional expectation, and one has $\im$-a.e. defined Koopman operator $\cHKoop=\cEstim_\RKHS^1$.
\end{remark}
{\bf Embedding Property.} Due to  \ref{eq:BK} we also have that RKHS $\RKHS$ can be embedded into $L^{\infty}_\im(\X)$, i.e.  for some $\epar\in(0,1]$
\[
 \cRKHSs{1} \hookrightarrow \cRKHSs{\epar}  \hookrightarrow L^{\infty}_\im(\X) \hookrightarrow \Lii,
\]
Now, according to \cite{Fischer2020}, if $\cTSs{\epar,\infty}\colon \cRKHSs{\epar} \hookrightarrow L^{\infty}_\im(\X)$ denotes the injection operator, its boundedness implies the polynomial decay of the singular values of $\cTS$, i.e. $\sigma_j^2(\cTS)\lesssim j^{-1/\epar}$, $j\in N$, and the following condition is assured
\begin{enumerate}[label={\rm \textbf{(KE)}},leftmargin=15ex]
\item\label{eq:KE} \emph{Kernel embedding  property}: there exists $\epar\in[\spar,1]$ such that 
\begin{equation}\label{eq:c_beta}
c_{\epar}:=\norm{\cTSs{\epar,\infty}}^2 =\esssup_{x\sim\im}\sum_{j\in N}\sigma^{2\epar}_j|\ell_j(x)|^2 <+\infty.
\end{equation}
\end{enumerate}

Finally, we make the following remark on finite-dimensional RKHS. 

\begin{remark}[\bf Finite-dimensional RKHS]
    When $\RKHS$ is finite dimensional, all spaces $[\RKHS]_\rpar$ are finite dimensional. Hence, $\range(\cKoop\cTS)\subset \range(\cTS)$ implies also $\range(\cTZ)\subset \range(\cTSs{\rpar})$ for every $\rpar>0$. Moreover, we can set $\epar$ and $\spar$ arbitrary close to zero.
\end{remark}

\begin{remark}[\bf Link to CME]
    Centering the feature map has been explored in the context of conditional mean embeddings (CME). The work most relevant to ours is \cite{Klebanov2020}, where one can find in-depth discussion on how kernel properties and centering affect the existence of $\cHKoop$. On the other hand, the results in \cite{Klebanov2020} are limited to the statistical consistency w.r.t. number of samples, while in this work we address finite sample learning rates and the impact of centering when learning dynamical systems.
\end{remark}

\subsection{Proof of Proposition 
\ref{prop:whithened}}

{\bf Embedding property and whitened feature maps.} 
The kernel embedding property \ref{eq:KE} allows one to estimate the norms of whitened centered feature maps  $\xi(x):=\cCreg^{-1/2}[\fm{x}-\kme{\im}]$, $\reg>0$, that play key role in deriving the learning rates,~\cite{Kostic2023b}.

\propWhithened*

\begin{proof}
We first observe that for every $\epar>0$ we have that
\begin{align*}
\norm{\xi(x)}^2 & = \sum_{j\in N}\scalarp{\cCreg^{-1/2}[\fm{x}-\kme{\im}],h_j}^2 = \sum_{j\in N}\frac{1}{\sigma_j^2+\reg} \scalarp{\fm{x}-\kme{\im},h_j}^2 = \sum_{j\in N}\frac{\sigma_j^{2(1-\epar)}}{\sigma_j^2+\reg} \frac{\scalarp{\fm{x}-\kme{\im},h_j}^2}{\sigma_j^{2}} \sigma_j^{2\epar} \\ 
& = \reg^{-\epar}\sum_{j\in N}\frac{(\sigma_j^2\reg^{-1})^{1-\epar}}{\sigma_j^2\reg^{-1}+1} \frac{\abs{h_j (x) - \EE_{X\sim\im}[h_j(X)]}^2}{\sigma_j^2} \sigma_j^{2\epar} \leq \reg^{-\epar}\sum_{j\in N} \frac{\abs{(\cTS h_j)(x)}^2}{\sigma_j^2} \sigma_j^{2\epar} = \reg^{-\epar}\sum_{j\in N} \abs{\ell_j(x)}^2\sigma_j^{2\epar},
\end{align*}
and, due to \eqref{eq:c_beta}, we obtain $\norm{\xi}_\infty^2 \leq \reg^{-\epar} c_{\epar}$.  On the other hand,  we also have that
\[
\norm{\xi}^2=\tr(\EE_{X\sim\im} [\xi(X)\otimes \xi(X)]) = \tr(\cCreg^{-1/2} \cCx \cCreg^{-1/2}) = \tr(\cCreg^{-1} \cCx),
\]
which is in uncentered case known as effective dimension of the RKHS $\RKHS$. Therefore, following the proof of  \cite[Lemma 11]{Fischer2020} for uncentered covariances, we show that the bound on the effective dimension  remains valid after centering with potentially improved bounds w.r.t. $\beta$. Namely, it holds that 
\begin{equation}
\label{eq:controlTrace_effective}
\tr(\cCreg^{-1}\cCx) = \sum_{j\in N} \frac{\sigma_j^2}{\sigma_j^2+\reg} \leq 
\begin{cases}
\frac{\scon^\spar}{1-\spar} \reg^{-\spar} &, \spar<1,\\
c_{\epar} \,\reg^{-1} &, \spar=1.
\end{cases}
\end{equation}
For the case $\spar=1$, it suffices to see that 
\[
\tr(\cCreg^{-1}\cCx)\leq \reg^{-1} \sum_{j\in N}\sigma_j^2 \norm{\ell_j}^2 = \reg^{-1} \int_{\X}\sum_{j\in N}\sigma_j^2 \abs{\ell_j(x)}^2\im(dx)\leq \reg^{-1} \esssup_{x\sim\im}\sum_{j\in N}\sigma_j^2 \abs{\ell_j(x)}^2\im(dx) = c_{\epar}\,\reg^{-1},
\]
while for $\spar<1$ we can apply the same classical reasoning as in the proof of Proposition 3 of \cite{caponnetto2007}.
\end{proof}

\subsection{Proof of Lemma \ref{lm:powerbound}}

\powerbound*

\begin{proof}

First note that w.l.o.g. we can assume that $\cHKoop\colon \Ker(\cTS)^\perp\to \Ker(\cTS)^\perp$. Next, since for $\cHKoop$ the peripheral spectrum is a subset of the point spectrum, let $\lambda$ be the leading eigenvalue of $\cHKoop$ and by $h\in\RKHS\setminus\{0\}$ its corresponding eigenvector. Then, since $\cKoop\cTS = \cTS \cHKoop$, we have that $\cKoop\cTS h=\lambda \cTS h$ and, due to $\cTS h\neq0$, we conclude that $\lambda$ is an eigenvalue of $\cKoop$, too. Therefore, $\rho(\cHKoop)=\abs{\lambda}\leq \rho(\cKoop)<1$.

Next, since $\kreiss(\cEEstim)\leq\pcon(\cEEstim)\leq (e/2)[\kreiss(\cEEstim)]^2$, it suffices to prove that for any two bounded operators $A$ and $\Delta$, one has that $\abs{\kreiss(A)-\kreiss(A+\Delta)}\leq \norm{\Delta} / (\dtoins(A) - \norm{\Delta})$. 

To that end, denote $B=A+\Delta$ and observe that $(B+zI)^{-1} - (A+zI)^{-1} = (B+zI)^{-1}\Delta(A+zI)^{-1}$, and, hence, 
\[
\abs{\norm{(B+zI)^{-1}} - \norm{(A+zI)^{-1}}} \leq \norm{(B+zI)^{-1}} \norm{\Delta} \norm{(A+zI)^{-1}},
\]
i.e.
\[
\abs{\norm{(B+zI)^{-1}}^{-1} - \norm{(A+zI)^{-1}}^{-1}} \leq\norm{\Delta}.
\]
Now, recalling definition of the Kreiss constant we have that
\[
\kreiss(B) =\sup_{\abs{z}>1}\frac{\abs{z}-1}{\norm{(B+zI)^{-1}}^{-1}} \leq \sup_{\abs{z}>1}\frac{\abs{z}-1}{\norm{(A+zI)^{-1}}^{-1} - \norm{\Delta}} = \sup_{\abs{z}>1}\frac{(\abs{z}-1)\norm{(A+zI)^{-1}}}{1- \norm{\Delta} /\norm{(A+zI)^{-1}}^{-1}}\leq \frac{\kreiss(A)}{1-\norm{\Delta} / \dtoins(A)}. 
\]
Since, we can show the lower bound in an analogous way, the proof is completed.  
\end{proof}

\subsection{Proof of Theorem \ref{prop:E_t_RRR}}

We provide additional details used in the proof of this result.

\begin{lemma}\label{lm:Gh-Ggamma-KRR}
Let Assumption \ref{eq:RC2} be satisfied. Then
\begin{equation}
    \norm{\cHKoop - \cRKoop}^2\leq a^2\reg^{\alpha-1}.
\end{equation}
\end{lemma}

\begin{proof}[Proof of Lemma 
\ref{lm:Gh-Ggamma-KRR}]
We have
\begin{align*}
\norm{\cHKoop - \cRKoop}^2 &= \norm{\cCreg^{-1}\cCxy- \cCx^{\dagger}\cCxy}^2 =  \norm{(\cCreg^{-1}- \cCx^{\dagger})\cCxy\cCxy^*(\cCreg^{-1} - \cCx^{\dagger})}\\
&\leq a^2 \norm{(\cCreg^{-1}- \cCx^{\dagger})\cCx^{1+\rpar}(\cCreg^{-1} - \cCx^{\dagger})}
= a^2\reg^{\alpha-1} \norm{\sum_{j\,:\, \sigma_j>0} \frac{(\gamma^{-1/2}\sigma_j)^{2(\rpar-1)}}{(1+(\gamma^{-1/2}\sigma_j)^2)^2} h_j\otimes h_j }^2\leq a^2\reg^{\alpha-1},
\end{align*}
where the last inequality holds due to $u^s \leq u+1$ for all $u\geq0$ and $s\in[0,1]$ and using that the norm of the orthogonal projector $\sum_{j\,:\, \sigma_j>0}  h_j\otimes h_j$ equals one. 
\end{proof}

\subsection{Forecasting with RRR.} We propose now to derive a result similar to Theorem \ref{prop:E_t_RRR} for the RRR estimator via an alternative argument which is valid for any $\rpar\in [1,2]$ in \eqref{eq:RC2}.
To that end, instead of Lemma \ref{lm:powerbound} we will use the following result based on the Carleman-type bound on the resolvent of compact operators, see e.g.~\cite{Bandtlow2004EstimatesFN}. 

\begin{restatable}{lemma}{powerbound2}
\label{lm:powerbound2}
 Let $\cEstim$ be a bounded linear operator on a Hilbert space such that $\srad(\cEstim)<1$. %
If $\cEstim$ has a finite rank $r$, then
\begin{equation}\label{eq:pcon}
\pcon(\cEstim):=\sup_{t\in\N_0}\Vert\cEstim^t\Vert\leq\frac{1}{2} \exp\left( \frac{2\, r\, \Vert \cEstim \Vert}{1\!-\!\srad(\cEstim)}+1\right).
\end{equation}
\end{restatable}

\begin{proof}
As before, we have that $p(\cEstim) \leq (e/2) [\kreiss(\cEstim)]^2$, but now, since  $\cEstim$ is finite rank, we bound $\kreiss(\cEstim)$ using Carleman- type inequality.  Namely,   due to \cite[Theorem 4.1]{Bandtlow2004EstimatesFN} for a trace-class operator $A$ it holds that 
$$\norm{(A-z I)^{-1}} \le \frac{1}{d(z, \Spec(A))} \exp\left(\frac{\norm{A}_{*}}{d(z, \Spec(A))}\right),$$
where $d(z, \Spec(A)):=\min_{\omega\in\Spec(A)}\abs{\omega-z}$ is the distance of $z\in\C$ to the spectrum of the operator $A$, and $\norm{\cdot}_{*}$ denotes nuclear norm. Since, $\norm{\cdot}_{*} \leq  r \ \norm{A}$ for $A$ of finite rank $r$, using that $\Spec(\cEstim)$  is contained in the open unit disk, we obtain  for $z\in\C$ s.t. $\abs{z}>1$
\[
    \norm{(z-{\cEstim})^{-1}} (\abs{z}-1) \le \frac{(\abs{z}-1)}{d(z, \Spec({\cEstim}))} \exp\left(\frac{\norm{\cEstim}_{*}}{d(z, \Spec({\cEstim}))}\right) \le \exp\left(\frac{r \norm{\cEstim}}{d(z, \Spec({\cEstim}))}\right),
\]
and, thus, 
$$\kreiss({\cEstim}) \le  \exp\left(\frac{r \norm{\cEstim}}{\inf_{\abs{z}>1}d(z, \Spec({\cEstim}))}\right)\leq \exp\left(\frac{r \norm{\cEstim}}{1-\srad(\cEstim)}\right)$$
\end{proof}

\begin{corollary}\label{cor:main_RRR}
Assume the operator $\cKoop$ is of finite rank $r$ for some $r\in\N$. 
Let \ref{eq:SD2} and \ref{eq:RC2} hold for some $\spar\in(0,1]$ and $\rpar\in[1,2]$, respectively. In addition, let $\cl(\range(\cTS))=\Lii$ and \ref{eq:BK} be satisfied. Let
\begin{equation*}%
    \reg\asymp n^{-\frac{1}{\rpar+\spar}}\,\text{ and }\,\rate^\star_n:= n^{-\frac{\rpar}{2(\rpar+\spar)}}.
\end{equation*}
Let $\delta\in(0,1)$. Then the forecasted distributions \eqref{eq:deflated_flow} based on $\cEEstim%
=\cERRR$ satisfy for $n$ large enough, with probability at least $1-\delta$, for any $t\geq 1$
\begin{equation*}
\label{eq:error_bound_dist_rrr}
\dnorm{\emt{t}-\mt{t}}\!\lesssim_{\bcon}\! %
e^{\frac{8r}{1-\rho(\cKoop)}} \left(  \left( a+ \frac{1}{\sigma_r^2(\cTZ)}\right)\,\rate_n^\star\,\ln(\delta^{-1})+ \sqrt{\frac{\ln \delta^{-1}}{n_0 \wedge n}}\right).
\end{equation*}
\end{corollary}

\begin{proof}
For brevity, we set $\cEEstim=\cERRR$ and $\cEstim=\cRRR$. Exploiting the definition and properties of the RRR model, we prove that 
there exists a constant $c\,{>}\,0$, depending only $r,\bcon,\spar$ such that 
for large enough $n\geq r$, with probability at least $1\,{-}\,\delta$ %
, the estimator $\cEEstim=\cERRR$ satisfies
$\norm{\cEEstim}\leq 2$, $1-\rho(\cEEstim)\geq \frac{1-\rho(\cEstim)}{2}$ and 
$
\cerror(\cEEstim)\lesssim_{\bcon} 
   \rate_n^\star\,\ln(\delta^{-1}).$

We first observed that 
$$
\cerror(\cEEstim) \leq \norm{\cKoop\cTS - \cTS \cRKoop} + \norm{\cTS(\cRKoop - \cEstim)} + \norm{\cTS (\cEEstim - \cEstim)}.
$$
Proposition 5 in \cite{Kostic2023b} 
and the condition $\cl(\range(\cTS))=\Lii$ immediately give $\norm{\cKoop\cTS - \cTS \cEstim}\leq a\reg^{\alpha/2}$, since for universal kernel, we have $\range(\cKoop\cTS) \subseteq \cl(\range(\cTS))$. 
By definition of $\cRKoop$ and since $\mathrm{rank}(\cKoop)=r$, we have $\cRRR=\cRKoop$. Hence $\norm{\cTS (\cRKoop- \cEstim)}=0 $.

We prove below that there exists a constant $c=c(\bcon)>0$ such that, for $n\geq r$ large enough
\begin{equation}\label{eq:variance_RRR_1}
\PP\left\{ \norm{\cTS (\cEEstim - \cEstim)}\leq c\, r^{\frac{2}{\spar}}\sqrt{\frac{1}{n\gamma^{\beta}}}\ln(\delta^{-1}) \right\} \geq 1-\delta.
\end{equation}

Combining the previous display with our control on the bias and using that $\reg\asymp n^{-\frac{1}{\rpar+\spar}}$, we get that
$$
\PP\left\{\cerror(\cEEstim) \lesssim_{\bcon} \left( a + r^{\frac{2}{\beta}}\right) n^{-\frac{\rpar}{2(\rpar+\spar)}}\log(\delta^{-1}) \right\} \geq 1-\delta.
$$
Define $\EB := \cECreg^{-1/2} \cECxy$ and let $\EP_r$ denote the orthogonal projector onto the subspace of leading $r$ right singular vectors of $\EB$. Then we have $\cERRR = \ERKoop\EP_r$. Hence, we have $\norm{\cERRR} \leq \norm{\cERKoop}$. Exploiting Proposition 16 in \cite{Kostic2023b}, we prove below that for $n$ large enough
\begin{equation}\label{eq:opnormGbounded}
\PP\left\{ \norm{\cEEstim}\leq 2 \right\} 
\geq 1-\delta.
\end{equation}

Next we apply Corollary 1 of \cite{Kostic2023b} to obtain that 
\begin{equation}\label{eq:spectralradiusGbounded}
    \PP\left\{ \rho(\cEEstim)\leq \rho(\cKoop)+ \rate_n^\star\,\ln(\delta^{-1}) \right\} 
\geq 1-\delta.
\end{equation}

Consequently on the same event, provided that $n$ is large enough, we deduce that 
$$
 1- \rho(\cEEstim)\geq \frac{1- \rho(\cKoop)}{2}
 $$
An elementary union bound combining the previous results and Lemma \ref{lm:powerbound2} below gives the result with probability at least $1-5\delta$. Up to a rescaling of the constant, we can replace $1-5\delta$ by $1-\delta$.

\end{proof}

\paragraph{Proof of Eq. \eqref{eq:variance_RRR_1}.}

We first define $\TB:=\cCreg^{-1/2} \cCxy$ and we recall that $\cCxy = \cTS^* \cKoop \cTS$. %
Applying Proposition 18 in \cite{Kostic2023b} gives, with probability at least $1-\delta$,
\begin{multline}\label{eq:variance_bound_rrr}
\norm{\cTS(\cRRR - \cERRR)} \leq \frac{c\,\rate_n^2(\reg,\delta/5)}{1-\rate_n^1(\reg,\delta/5)}
+ \frac{\sigma_1(\TB)}{\sigma_r^2(\TB) -\sigma_{r+1}^2(\TB)} \,\frac{(c^2-1)\,\rate_n(\delta / 5) + c^2\,(\rate_n^2(\reg,\delta/5))^2}{(1-\rate_n^1(\reg,\delta/5))^2},
\end{multline}
where %
$c:=1+\rcon\,\bcon^{(\rpar-1)/2}$,
\begin{equation}\label{eq:eta}
\rate_n(\delta) := \frac{4\bcon}{3n} \mathcal{L}(\delta) + \sqrt{\frac{2\norm{\cCx}}{n}\mathcal{L}(\delta)}\quad\text{ and }\quad \mathcal{L}(\delta):= \log\frac{4\tr(\cCx)}{\delta\,\norm{\cCx}},
\end{equation}
\begin{equation}\label{eq:reg_eta1}
\rate_n^1(\reg,\delta) := \frac{4\econ}{3n\reg^{\epar}} \mathcal{L}^1(\reg,\delta)+ \sqrt{\frac{2\,\econ}{n\,\reg^{\epar}}\mathcal{L}^1(\reg,\delta)},
\end{equation}
with
\[
\mathcal{L}^1(\reg,\delta):=\log \frac{4}{\delta} + \log\frac{\tr(\cCreg^{-1}\cCx)}{\norm{\cCreg^{-1}\cCx}}, 
\]
and
\begin{equation}\label{eq:reg_eta2}
\rate_n^2(\reg,\delta) := 4\,\sqrt{2\,\bcon}\,\left(\sqrt{\frac{\tr(\cCreg^{-1}\cCx)}{n}} + \frac{\sqrt{\econ}}{n\reg^{\epar/2}}\right) \,\log\frac{2}{\delta}.
\end{equation}
Using \eqref{eq:controlTrace_effective} and elementary computations, we get that the dominating term in \eqref{eq:variance_bound_rrr} is of the order $\sqrt{\frac{1}{n\gamma^{\beta}}}\ln(\delta^{-1})$ since $\epar\geq \spar$. In addition, for $\TB:=\cCreg^{-1/2} \cCxy$, we have $\sigma_{r+1}(\TB)=0$ since $\mathrm{rank}(\cCxy)\leq \mathrm{rank}(\cKoop)=5$. Finally, Proposition 6 of \cite{Kostic2023b} and our choice of $\reg$ guarantees for $n$ large enough that $\sigma_r^2(\TB)\geq \sigma_r^2(\cTZ)-\rcon^2\,\bcon^{\rpar/2}\,\reg^{\rpar/2} \geq \sigma_r^2(\cTZ)/2 >0$. 

\paragraph{Proof of Eq. \ref{eq:opnormGbounded}.} Proposition 16 in \cite{Kostic2023b} guarantees with probability at least $1-\delta$
$$
\norm{\cERKoop}\leq  \frac{1+\rate^3_n(\reg,\delta / 2)}{1-\rate^3_n(\reg,\delta / 2)},
$$
where 
\begin{equation}\label{eq:reg_eta3}
\rate_n^3(\reg,\delta) := 4\,\sqrt{2\,\bcon}\,\left(\sqrt{\frac{\tr(\cCreg^{-2}\cCx)}{n}} + \frac{\sqrt{\econ}}{n\reg^{(1+\epar)/2}}\right) \,\log\frac{2}{\delta}.
\end{equation}

Using again \eqref{eq:controlTrace_effective} and the fact that $\epar\geq \spar$, we deduce that
$$
\rate_n^3(\reg,\delta) \lesssim \sqrt{\bcon}  \sqrt{\frac{1}{n \gamma^{1+\spar}}}\log(2\delta^{-1}).
$$
With our choice of $\gamma$ and for $n$ large enough such that $\rate_n^3(\reg,\delta/2) <1/4$, we get 
\[
\PP\left\{ \norm{\cERKoop}\leq  \frac{1+\rate^3_n(\reg,\delta / 2)}{1-\rate^3_n(\reg,\delta / 2)}\leq 2 \right\} 
\geq 1-\delta.
\]

\section{Experimental Details}
\label{app:exp_details}
In both experiments, the Reduced Rank Regression estimator was implemented using the reference code from~\cite{Kostic2022} available at \href{https://github.com/CSML-IIT-UCL/kooplearn}{https://github.com/CSML-IIT-UCL/kooplearn}. The experiments were run on a workstation equipped with an Intel(R) Core\texttrademark i9-9900X CPU @ 3.50GHz, 48GB of RAM and a NVIDIA GeForce RTX 2080 Ti GPU. All experiments have been implemented in Python 3.11.

For the {\em Ornstein-Uhlenbeck} experiment we sampled the process every $dt = 0.05$. The estimators were trained with $250$ observations sampled independently from the invariariant distribution, while the initial distribution used to evaluate the MMD was sampled 1000 times. Each experiment has been repeated 100 times independently, and the hyperparameter were tuned on a validation set of 500 points sampled from the invariant distribution. 

The code to reproduce the experiments will be open sourced.

\end{document}